\documentclass[final,breaklinks]{colt2020-plain}
\usepackage{times}
\usepackage{booktabs}
\usepackage{multirow}
\usepackage{amsmath,amssymb,rotating,graphicx,rotating,float}
\usepackage{url}

\usepackage{color}

\newlength{\dhatheight}

\newcommand{\Px}{\mathcal{P}}

\newcommand{\s}{\mathfrak{s}}

\newcommand{\SC}{\mathcal{M}}

\newcommand{\vc}{{\rm d}}
\newcommand{\covcvague}{{\rm k}}
\newcommand{\covc}{{\rm k}_w}
\newcommand{\covcmod}{{\rm k}_o}
\newcommand{\covcproj}{{\rm k}_p}

\newcommand{\Log}{{\rm Log}}

\newcommand{\Alg}{\mathbb{A}}

\renewcommand{\epsilon}{\varepsilon}
\newcommand{\eps}{\varepsilon}

\newcommand{\X}{\mathcal X}

\newcommand{\Y}{\mathcal Y}
\newcommand{\alg}{\mathbb{A}}
\renewcommand{\H}{\mathbb C}

\newcommand{\Z}{\mathcal{Z}}
\newcommand{\target}{f^{\star}}

\newcommand{\er}{{\rm er}}
\newcommand{\ER}{{\rm ER}}

\DeclareSymbolFont{bbold}{U}{bbold}{m}{n}
\DeclareSymbolFontAlphabet{\mathbbold}{bbold}
\newcommand{\ind}{\mathbbold{1}}

\newcommand{\C}{\mathbb{C}}
\newcommand{\A}{\mathcal{A}}

\newcommand{\sign}{{\rm sign}}

\renewcommand{\P}{{\rm Pr}}
\newcommand{\nats}{\mathbb{N}}
\newcommand{\reals}{\mathbb{R}}

\newcommand{\E}{\mathbb{E}}
\newcommand{\Var}{{\rm Var}}

\newcommand{\ignore}[1]{}
\newcommand{\oldstuff}[1]{}

\newcommand{\new}[1]{}
\newcommand{\sh}[1]{}

\newsavebox{\savepar}
\newenvironment{bigboxit}{\begin{center}\begin{lrbox}{\savepar}
\begin{minipage}[h]{5.2in}
\normalfont
\begin{flushleft}}
{\end{flushleft}\end{minipage}\end{lrbox}\fbox{\usebox{\savepar}}
\end{center}}

\makeatletter
\newcommand{\vast}{\bBigg@{3}}
\newcommand{\Vast}{\bBigg@{4}}
\makeatother

\renewenvironment{proof}[1][]{\par\noindent{\bf Proof #1\ }}{\hfill\BlackBox\\[2mm]}

\title{Proper Learning, Helly Number, and an Optimal SVM Bound}

\coltauthor{%
 \Name{Olivier Bousquet} \Email{obousquet@google.com} \\ \addr Google Research, Brain Team 
 \AND
\Name{Steve Hanneke} \Email{steve.hanneke@gmail.com} \\ \addr Toyota Technological Institute at Chicago 
 \AND
 \Name{Shay Moran} \Email{shaymoran@google.com} \\ \addr Google Research, Brain Team 
  \AND
 \Name{Nikita Zhivotovskiy} \Email{zhivotovskiy@google.com} \\ \addr Google Research, Brain Team
}

\begin{document}
\maketitle

\begin{abstract}
The classical PAC sample complexity bounds are stated for any Empirical Risk Minimizer (ERM) 
    and contain an extra multiplicative logarithmic factor $\log \frac{1}{\eps}$ 
    which is known to be necessary for ERM in general. 
    It has been recently shown by \citet*{hanneke:16a} that 
    the optimal sample complexity of PAC learning for any VC class $\H$ does not include this log factor and 
    is achieved by a particular \emph{improper learning algorithm}, 
    which outputs a specific majority-vote of hypotheses in $\H$.
    This leaves the question of when this bound can be achieved by \emph{proper learning algorithms},
    which are restricted to always output a hypothesis from $\H$.

In this paper we aim to characterize the classes for which the optimal sample complexity can be achieved by a proper learning algorithm.
    We identify that these classes can be characterized by the \emph{dual Helly number},
    which is a combinatorial parameter that arises in discrete geometry and abstract convexity.
    In particular, under general conditions on $\H$, 
    we show that the dual Helly number is bounded 
    if and only if there is a proper learner that obtains the optimal dependence on $\epsilon$.
    
    As further implications of our techniques we resolve a long-standing open problem posed by \citet*{vapnik:74} on the performance of the \emph{Support Vector Machine} in $\mathbb{R}^n$ by proving that the sample complexity of SVM in the realizable case is 
    \[
    \Theta\!\left( \frac{n}{\eps} + \frac{1}{\eps}\log\frac{1}{\delta} \right).
    \]
    This gives the first optimal PAC bound for Halfspaces in $\mathbb{R}^n$ achieved by a proper learning algorithm, 
    and moreover is computationally efficient.
\end{abstract}


\section{Introduction}

In the literature on the theory of PAC learning, 
there has been much work discussing the important 
distinction between \emph{proper} vs \emph{improper} 
learning algorithms, where a proper learner is required 
to output a hypothesis from the concept class being learned, 
while an improper learner may output any classifier, not necessarily 
in the class.  Most of this literature has focused on the \emph{computational} 
separations between proper and improper learning (see e.g.,~\citealp*{kearns:94}).
However, it is also interesting to consider the effect on \emph{sample complexity} 
of proper vs improper learning.
While the optimal sample complexity of PAC learning was recently 
resolved by \citet*{hanneke:16a}, the proposed learning algorithm 
is \emph{improper}: constructing its classifier based on a majority 
vote of well-chosen classifiers from the concept class.
Furthermore, it follows from arguments analogous to the work of \citet*{daniely:14}
that the optimal sample complexity of PAC learning is sometimes \emph{not achievable} 
by proper learners (see also our Theorem~\ref{thm:sometimes-tight}).
While the question of characterizing the best sample complexity achievable by proper learners  
has been resolved for several special-case concept classes \citep*[e.g.,][]{auer:07,darnstadt:15,hanneke:16b},
the best known \emph{general} results on the sample complexity of proper learning
in the prior literature are still only the results that hold for \emph{all} empirical risk minimization (ERM)  
algorithms \citep*{vapnik:74,blumer:89,hanneke:16b,zhivotovskiy2018localization}. 
However, it is well known that there are many classes where specific proper learners
can achieve better sample complexities (by a log factor) than the worst ERM learner \citep*{auer:07}.
Thus, it is important to go beyond the general analysis of ERM learners 
if we are to truly understand the best sample complexity achievable by 
proper learners, in a general analysis.

In the present work, we aim to provide such a general analysis of the 
sample complexity of proper learning, applicable to \emph{every} concept class,
by identifying the relevant combinatorial complexity measures of the class.
We specifically find that a quantity called the \emph{dual Helly number}
(previously proposed by \citealp*{kane19a} under the name \emph{coVC dimension})
is of critical importance.  
In particular, when the dual Helly number is finite, the logarithmic factor 
in the well-known sample complexity bounds for ERM \citep*{vapnik:74} 
may be replaced by a bounded quantity.
The proper learning algorithm that the upper bound holds for is a modification 
of the optimal PAC learner of \citet*{hanneke:16a}, but modified 
in several steps so that it remains proper.  

As a further implication of the techniques we develop, we find
that in the case of learning Halfspaces in $\reals^n$, the well-known \emph{support vector machine} (SVM) 
learning algorithm achieves the optimal sample complexity 
$\Theta\!\left( \frac{n}{\eps} + \frac{1}{\eps}\log\frac{1}{\delta} \right)$.
This resolves a question that appeared in the seminal work of \citet*{vapnik:74}.
Moreover, this also provides the first proof that Halfspaces are properly learnable with the optimal sample complexity: that is, sample complexity of the form $\frac{n}{\eps}+\frac{1}{\eps}\log\frac{1}{\delta}$.
As a further implication, we find that Maximum classes  
of any given VC dimension $\vc$ are also properly learnable with optimal 
sample complexity $\Theta\!\left( \frac{\vc}{\eps} + \frac{1}{\eps}\log\frac{1}{\delta} \right)$.

\begin{figure}
\begin{center}
\renewcommand{\arraystretch}{1.2}
\begin{tabular}{|c|c|c| }
\hline
\multicolumn{3}{|c|}{Bounds on the sample complexity of PAC learning} \\
\hline
\multirow{2}*{Improper Learning} & \multirow{2}*{$\Theta\!\left(\frac{\vc}{\eps} + \frac{1}{\eps}\log \frac{1}{\delta} \right)$} & \citealp*{hanneke:16a}\\
 & & \citealp{ehrenfeucht:89}\\
\hline 
\multirow{2}*{Any ERM} & $O\!\left(\frac{\vc}{\eps} \log(\frac{1}{\eps} \land \frac{\s}{\vc}) + \frac{1}{\eps}\log \frac{1}{\delta}\right)$ & \citealp*{hanneke:16b}\\
 & $\Omega\!\left(\frac{\vc}{\eps} + \frac{1}{\eps}\log(\frac{1}{\eps} \land \s) + \frac{1}{\eps}\log \frac{1}{\delta}\right)$ & \citealp*{vapnik:74} \\
\hline
\multirow{2}{*}{Proper Learning} & $O\!\left(\frac{\vc \covcvague^2}{\eps} \log(\covcvague) + \frac{\covcvague^2}{\eps}\log \frac{1}{\delta}\right)$ & \multirow{2}*{\color{blue} New results in this work.} \\
& $\Omega\!\left(\frac{\vc}{\eps} + \frac{1}{\eps}\log(\covcvague) + \frac{1}{\eps}\log \frac{1}{\delta}\right)$ & \\
\hline
\multirow{1}{*}{SVM / Halfspaces in $\reals^n$} & $\Theta\!\left( \frac{n}{\eps} + \frac{1}{\eps}\log\frac{1}{\delta} \right)$ & {\color{blue} New result in this work.}\\ 
\hline
\multirow{1}{*}{Maximum Class (Proper)} & $\Theta\!\left( \frac{\vc}{\eps} + \frac{1}{\eps}\log\frac{1}{\delta} \right)$ & {\color{blue} New result in this work.}\\ 
\hline
\end{tabular}
\end{center}
\caption{Summary of results on the sample complexity of $(\epsilon,\delta)$-PAC learning, 
along with our new results. $\vc$ denotes the \emph{VC dimension} \citep*{vapnik:71}, 
$\s$ the \emph{star number} \citep*{hanneke:15b}, 
and $\covcvague$ the \emph{dual Helly number} \citep*{kane19a} discussed in this article. 
Specific definitions, conditions, and ranges of parameters for which the results hold 
are discussed below. 
}
\label{fig:table}
\end{figure}

The known results on sample complexity are summarized in Figure~\ref{fig:table} 
along with a (rough) statement of our new results.

\paragraph{Notation.} To begin the formal discussion, we introduce some basic notation.
Fix a space $\X$ equipped with a $\sigma$-algebra specifying the measurable subsets.
Let $\Y = \{-1,1\}$ denote the \emph{label space}.  A \emph{classifier} is any 
measurable function $h : \X \to \Y$, and a \emph{concept class} is any set $\H$ of classifiers.
To focus on nontrivial cases here, we will always suppose $|\H| \geq 3$.
A \emph{learning algorithm} $\A$ maps any sequence (data set) $\{(x_1,y_1),\ldots,(x_n,y_n)\}$ in $\X \times \Y$, 
of any length $n$, to a classifier $\hat{h}_n$; the map $\A$ may include randomization.
A learning algorithm $\A$ is called \emph{proper} (for~$\H$) 
if $\hat{h}_n$ is always an element of $\H$, for all possible data sets.
Otherwise $\A$ is called \emph{improper}.

In the PAC learning problem, there is a \emph{data distribution} $\Px$ (a probability measure on $\X$), 
and a \emph{target concept} $\target \in \H$.  For any classifier $h$, define 
$\er_{\Px}(h;\target) = \Px( x : h(x) \neq \target(x) )$.
There is an i.i.d.\ sequence of $\Px$-distributed 
random variables $X_1(\Px),X_2(\Px),\ldots$.  
When $\Px$ and $\target$ are clear from the context, we simply write $\er(h)$ and $X_1,X_2,\ldots$.
Define $a \land b = \min\{a,b\}$
for $a,b \in \reals$.
Generally, for any sequence $x_1,x_2,\ldots$ and any $f : \X \to \Y$, 
we use the notation $1:n = \{1,\ldots,n\}$, 
$x_{1:n} = \{x_1,\ldots,x_n\}$, and $(x_{1:n},f(x_{1:n})) = \{(x_1,f(x_1)),\ldots,(x_n,f(x_n))\}$.
The \emph{sample complexity}, the central quantity of study in this work, is defined as follows.

\begin{definition}
\label{def:sample-complexity}
For any $\epsilon,\delta \in (0,1)$, the \emph{sample complexity} of $(\epsilon,\delta)$-PAC learning, 
denoted
$\SC(\epsilon,\delta)$, is defined as the smallest $n \in \nats$ for which there exists 
a learning algorithm $\A$ such that, for every data distribution $\Px$ and every $\target \in \H$, 
the (random) classifier $\hat{h}_n = \A((X_{1:n},\target(X_{1:n})))$ satisfies 
\begin{equation*}
\P\!\left( \er(\hat{h}_n) \leq \epsilon \right) \geq 1-\delta.
\end{equation*}
The sample complexity of $(\epsilon,\delta)$-PAC \emph{proper} learning,
denoted by $\SC_{{\rm prop}}(\epsilon,\delta)$, is defined identically, 
except that the learning algorithm $\A$ is required to be \emph{proper}: 
it always outputs an element of $\H$.
\end{definition}

A fundamental quantity in characterizing the sample complexity is the \emph{VC dimension} \citep*{vapnik:71}.
We say $\H$ \emph{shatters} a sequence of points $x_{1:n} \in \X^n$ if 
$\forall y_{1:n} \in \Y^n$, $\exists h \in \H$ with $h(x_{1:n}) = y_{1:n}$.
The VC dimension of $\H$, denoted by $\vc$, is the largest $n \in \nats$ for which 
there exists a sequence $x_{1:n}$ shattered by $\H$; 
otherwise if no such largest $n$ exists, define $\vc = \infty$.

The sample complexity of (unrestricted) PAC learning was recently proven by \citet*{hanneke:16a} to satisfy 
$\SC(\epsilon,\delta) = \Theta\!\left( \frac{\vc}{\epsilon} + \frac{1}{\epsilon}\log\!\left(\frac{1}{\delta}\right)\right)$, 
resolving a gap between the previously-known lower bound of this form 
(from \citealp*{vapnik:74, ehrenfeucht:89})
and previous suboptimal upper bounds \citep*{vapnik:74,blumer:89,haussler:94,simon:15}.
However, it is interesting to note that the optimal learning algorithm proposed by \citet*{hanneke:16a} 
is \emph{improper}.  

Most of the work on the sample complexity of proper learning is based on the 
fact (due to \citealp*{vapnik:74}, \citealp*{blumer:89}) that any learning algorithm that outputs any $\hat{h} \in \H$ 
making no mistakes on the training data (called \emph{empirical risk minimization}, or ERM) is 
guaranteed to achieve a sample complexity 
$O\!\left( \frac{\vc}{\epsilon} \log\!\left(\frac{1}{\epsilon}\right)+\frac{1}{\epsilon}\log\!\left(\frac{1}{\delta}\right)\right)$.
This bound has been refined in some special cases \citep*{hanneke:16b,zhivotovskiy2018localization}, 
but it is known that it cannot generally be improved while still holding for all ERM learners 
\citep*{blumer1989learning,haussler:94,auer:07}.
On the other hand, for some special types of concept classes,
it was observed that $\SC_{{\rm prop}}(\epsilon,\delta) = 
\Theta\!\left( \frac{\vc}{\epsilon} + \frac{1}{\epsilon}\log\!\left(\frac{1}{\delta}\right)\right)$: 
that is, proper learning is sometimes optimal.
For instance, this was shown for any \emph{intersection-closed} concept class, 
where the optimal sample complexity is achieved by a proper learner known as the \emph{Closure} algorithm 
\citep*{auer:07,darnstadt:15,hanneke:16b}.
For the class of Halfspaces on $\reals^n$,
\citet*{vapnik:74} found that the support vector machine (SVM) classifier
(which is a proper learner)
achieves the optimal dependence on $\vc$ and $\epsilon$ 
for obtaining expected error $\epsilon$,
and they essentially asked the question of whether the exact optimal form
$\Theta\!\left(\frac{\vc}{\epsilon}+\frac{1}{\epsilon}\log\!\left(\frac{1}{\delta}\right)\right)$ 
for PAC learning is achieved by SVM. 
This question has remained open since then, with a number of 
works investigating the question 
\citep*[e.g.,][]{blumer1989learning, balcan:13,zhivotovskiy2017,hanneke:19b,long2020}.
We answer this question affirmatively, finding that indeed the SVM classifier 
achieves the optimal sample complexity for learning Halfspaces.

Related results are known for the multi-class setting (where we may have $|\Y| > 2$, or even infinite $\Y$).
In this case, \citet*{daniely2015multiclass} showed that different ERMs may have strikingly different sample complexities, 
and \citet*{daniely:14} showed that there exist classes (with $|\Y|=\infty$) that are learnable but \emph{not} properly learnable 
(i.e., $\SC(\epsilon,\delta)<\infty$ but $\SC_{{\rm prop}}(\epsilon,\delta)=\infty$).  A similar separation between 
proper and improper learnability was also recently shown by \citet*{montasser:19} for the problem of 
learning with adversarial robustness guarantees.
Of course, these kinds of striking separations cannot happen in the binary classification setting ($|\Y|=2$) 
studied here, since the 
aforementioned result of \citet*{vapnik:74} shows that ERM learners obtain sample complexities 
that are at most suboptimal by a factor $O(\log \frac{1}{\epsilon})$.  However, the argument used in the 
proofs of \citet*{daniely:14} and \citet*{montasser:19} can be adapted to show that this logarithmic factor is 
sometimes \emph{necessary}: that is, that there exist classes of any given VC dimension $\vc$ 
for which $\SC_{{\rm prop}}(\epsilon,\delta) = \Omega( \frac{\vc}{\epsilon} \log \frac{1}{\epsilon} + \frac{1}{\epsilon} \log \frac{1}{\delta} )$.\footnote{This result is unpublished, and essentially folklore, discovered independently 
by several different people familiar with the argument of \citet*{daniely:14}.}

One of the aims of this paper is to connect these scattered observations and explain the properties of the class $\H$ 
responsible for the optimality or sub-optimality of proper PAC learners.

\paragraph{Structure of the paper and main contributions}
\begin{itemize}
\item Section \ref{sec:definitions} contains the definition of the dual Helly number and two variants of it: 
    the \emph{hollow star} number and the \emph{projection number}. 
    We show that these three parameters coincide whenever they are finite 
    and present some (infinite) classes which demonstrates subtle cases in which these parameters differ.
    \item Section \ref{sec:upper-bound} contains an upper bound for proper learning classes with bounded dual Helly number.
    Specifically, we show that if the \emph{projection number} is bounded, there is a particular ERM having a better 
    sample complexity than arbitrary ERM. We also provide a proper learning algorithm achieving an even more-improved sample complexity, but this proper  
    algorithm is not necessarily an ERM.
    \item Section \ref{sec:lower-bound} contains a lower bound on the sample complexity of proper algorithms when the 
    hollow star number is large.
    \item Section \ref{sec:SVM} presents a new upper bound for \emph{stable compression schemes}: 
    compression schemes whose choice of compression set 
    is unaffected by removing points not in the 
    compression set.
    In particular, as SVM can be viewed as a stable 
    compression scheme, our result implies that SVM requires only $O\!\left(\frac{n}{\eps} + \frac{1}{\eps}\log \frac{1}{\delta} \right)$ 
    examples in order to $(\epsilon,\delta)$-PAC learn $n$-dimensional halfspaces.
    This resolves a long-standing open problem from  \citet*{vapnik:74}. 
    As a further implication of our general result for 
    stable compression schemes, we also find that 
    all Maximum classes are 
    properly learnable with optimal sample 
    complexity.

\end{itemize}

\section{The dual Helly number}
\label{sec:definitions}

Helly's Theorem is a fundamental result about convex sets~\citep{Helly1923}.
    It asserts that for any sequence of convex sets $C_1\ldots C_m\subseteq \mathbb{R}^n$ 
    such that $\cap_{i=1}^{m} C_i = \emptyset$ there is a  subsequence $C_{i_1}\ldots C_{i_k}$,
    with $k\leq n+1$ such that $\cap_{j=1}^{k} C_{i_j} = \emptyset$.
    This notion has been studied more abstractly in various settings (see, e.g.,~\citealp*{levi1951helly,danzer1963helly});
    it is defined in an abstract manner as follows:
    let $\mathcal{F}$ be a family of subsets over a domain $\X$.
    The {\it Helly number} of $\mathcal{F}$ is the minimum integer ${\rm k}$
    such that whenever $\mathcal{C} \subseteq \mathcal{F}$ 
    is a collection of sets 
    whose intersection is empty then there is a 
    subset $\mathcal{C}^{\prime} \subseteq \mathcal{C}$
    of size at most ${\rm k}$ whose intersection is empty.
    That is, the empty intersection of the entire collection $\mathcal{C}$ is witnessed
    by a subset of size at most ${\rm k}$.

\paragraph{The Dual Helly Number of a Class.}
We now adapt the Helly number (in a dual manner) to our context.
The obtained parameter\footnote{We note that an equivalent parameter was introduced by the name ``coVC dimension'' by \citet*{kane19a},
where it was used to characterize proper learning in a distributed setting.} plays a central role in our characterization of the proper sample complexity. For any $S \subseteq \X \times \Y$, 
define $\H[S] = \{ h \in \H : \forall (x,y) \in S, h(x)=y \}$.

\begin{definition}[The dual Helly number]
\label{def:witness}
Define the dual Helly number of $\H$, denoted by $\covc$, 
as the smallest integer $k$ such that, for any $S \subseteq \X \times \Y$ 
such that $\H[S] = \emptyset$, 
there is a set $W \subseteq S$ with $|W| \leq k$ such that
$\H[W] = \emptyset$.
That is, for any unrealizable set of points, there is an unrealizable subset of size at most $k$.
If no such $k$ exists, we define $\covc = \infty$.
\end{definition}
Observe that the dual Helly number is precisely the Helly number of the following family $\mathcal{F}$ 
which is defined on the dual space:
for each example $(x,y)\in \X\times \Y$ let $\H_{(x,y)}:= \{h\in \H : h(x)=y\}$ 
and let $\mathcal{F}:=\{\H_{(x,y)} : (x,y)\in\X\times \Y\}$.
This definition is also related to the notion of a \emph{teaching set} \citep*{goldman:95}: 
recall that $W\subseteq X$ is a teaching set for $h$ with respect to $\H$ if there exists no $h'\in \H\setminus\{h\}$ 
which agrees with $h$ on $W$. 
In particular, observe that any $h\notin \H$ has a teaching set with respect to $\H$ of size at most $\covc(\H)$.

We proceed with the second definition. 
We will need the following notation:
two sequences (or samples) $(x_1,y_1),\ldots,(x_k,y_k)\in \X\times \Y$ and 
$(x_1',y_1'),\ldots, (x_k',y_k')\in \X\times \Y$ are said to be \emph{neighbors}
if $x_i=x_i'$ for all $i \in 1:k$ 
and there exists exactly one $j \in 1:k$
such that $y_j\neq y_j'$.

\begin{definition}[The hollow star number]
\label{def:hollowstar}
Define the \emph{hollow star number} of $\H$, denoted by  $\covcmod$,
as the largest integer $k$ such that there is a sequence $S=((x_1,y_1),\dots (x_k,y_k))\in (\X\times\Y)^k$
which is \underline{not} realizable by $\H$ (i.e.\ $\H[S]=\emptyset$), 
however every sequence $S'$ which is a neighbor of $S$ is realizable by $\H$ (i.e.,\ $\H[S']\neq\emptyset$).  
If no such largest $k$ exists, define $\covcmod = \infty$.
\end{definition}

We refer to any unrealizable sequence $S \in (\X \times \Y)^*$, 
such that every neighbor $S'$ of $S$ is realizable, 
as a \emph{hollow star set}.  Thus, $\covcmod$ is the 
size of the largest finite hollow star set, or $\infty$ if 
there exist hollow star sets of unbounded finite sizes.  
Note that (equivalently) any hollow star set 
$S = \{(x_1,y_1),\ldots,(x_k,y_k)\}$ 
has the properties that $\H[S]=\emptyset$ 
and $\forall i \leq k$, $\exists h_i \in \H$ 
s.t.\ $\{ j : h_i(x_j) \neq y_j \} = \{ i \}$.

The name \emph{hollow star number} is chosen to stress the similarity with the \emph{star number} defined 
by \citet*{hanneke:15b}.
The \emph{star number} $\s$ is the maximum size of a \underline{realizable} sequence $S$ such that every sequence $S'$ which is a neighbor of $S$ is also realizable.  Thus, the definitions of star number 
and hollow star number differ 
\emph{only} in whether $S$ is required to be 
realizable or unrealizable.

The star number $\s$ was defined by \citet*{hanneke:15b} to characterize the PAC sample complexity of 
active learning, 
and was shown by \citet*{hanneke:16b} to also 
characterize the general rate of convergence of ERM 
(see Section~\ref{sec:star-ERM} below).
Interestingly, it can be shown \citep*{hanneke:15b} that the star number upper bounds 
the size of the \emph{teaching set} of any classifier.  
It also has many other 
connections to various quantities arising in the 
learning theory literature 
\citep*{hanneke:15b,hanneke:16b}.
From the definitions, we immediately have the simple inequality:
\begin{equation*}
    \covcmod - 1 \le \s.
\end{equation*}
However, as we discuss below, while the classes $\H$ 
having $\s < \infty$ are very limited, 
classes with $\covcmod < \infty$ are 
far more common (e.g., Halfspaces). Thus, 
this one small difference in the definition, 
requiring the star's \emph{center} $S$ to be unrealizable, 
significantly impacts the value of the quantity.\\

Our final definition is slightly more involved but it will play a key role in our upper bound. 
For a finite (multiset) $\H^{\prime} \subseteq \H$ let $\rm{Majority}(\H^{\prime}):\X\to\{0,1,?\}$ denote 
the majority-vote classifier defined by:
\[
\rm{Majority}(\H^{\prime})(x) =
\begin{cases}
0 & \bigl\lvert\{c\in \H^{\prime} : c(x) = 0\}\bigr\rvert > \frac{\lvert \H^{\prime}\rvert}{2},\\
1 & \bigl\lvert\{c\in \H^{\prime} : c(x) = 1\}\bigr\rvert > \frac{\lvert \H^{\prime}\rvert}{2},\\
? &\text{else.}
\end{cases}
\]
For $l \geq 2$, define the set $\mathcal{X}_{\H^{\prime}, l}\subseteq \X$  of all the points $x$ on which less than $\frac{1}{l}$-fraction of all classifiers in $\H^{\prime}$ disagree with the majority. That is,
\begin{equation}
\label{eq:majoritylabel}
\mathcal{X}_{\H^{\prime}, l} = \left\{x \in \X: \sum_{h \in \H^{\prime}} \ind[ h(x) \neq h_{{\rm maj}}(x) ] < \frac{|\H^{\prime}|}{l}\right\},
\end{equation}
where $h_{{\rm maj}} = \rm{Majority}(\H^{\prime})$.

\begin{definition}[The projection number]
\label{def:projection}
Define the \emph{projection number} of $\H$, denoted by $\covcproj$, as the 
smallest integer $k \geq 2$ such that, for any finite 
multiset 
$\H^{\prime} \subseteq \H$ there exists $h \in \H$ that agrees with $\rm{Majority}(\H^{\prime})$ on the entire set $\mathcal{X}_{\H^{\prime}, k}$.
If no such integer $k$ exists, define $\covcproj = \infty$.
\end{definition}

\noindent This definition allows us to ``project'' the majority vote of any classifiers in $\H$ to the class $\H$.  Define
\[
{\rm{Proj}_{\H}}(\H^{\prime})\; \textrm{is any element in}\ \left\{h \in \H: h(x) = {\rm Majority}(\H^\prime),\; \textrm{for all}\; x \in \mathcal{X}_{\H^{\prime}, \covcproj}\right\}.
\]
The set used in this definition is always non-empty (by 
definition of $\covcproj$) when $\covcproj < \infty$. 

Following \citet*{kane19a}, we say a class $\H$ is ``closed'' 
if every $S \subseteq \X \times \Y$ with $\H[S]=\emptyset$ 
has a finite subset $S' \subseteq S$ with $\H[S']=\emptyset$.
The following lemma connects these quantities.  Its proof is included in Appendix~\ref{sec:omittedexamples}.
\begin{lemma}
\label{lem:all-the-hellys}
\begin{itemize}
\item $\covcmod \leq \covcproj \leq \covc$.
\item If $\covc < \infty$ or $\H$ is closed, 
then $\covcmod = \covcproj = \covc$.
\end{itemize}
\end{lemma}

\begin{remark}
Certainly if $\X$ or $\H$ is \emph{finite} then $\covc < \infty$, 
so that $\covcmod = \covcproj = \covc$.  However, in the general case, 
there are examples where each of these inequalities can be strict,  
due to one quantity being infinite and another finite.
We discuss such examples in Section~\ref{subsec:examples} below.
\end{remark}

\subsection{Some Examples}
\label{subsec:examples}

Let us now argue that many classes of interest have 
finite values for these complexity measures.
Where appropriate, details of the examples   
are provided in Appendix \ref{sec:omittedexamples}. 
We start with the Halfspaces.
\begin{example}
\label{ex:dualhellyforhalfspaces}
Let $\H$ be a class induced by halfspaces in $\mathbb{R}^{n}$.
Then $\covcmod = \covcproj = n + 2 = \vc+1$.
\end{example}

Another simple example is the case of \emph{intersection closed classes}. 
The class $\H$ is intersection closed if the set $\{\{x: h(x) = +1\} : h \in \H\} $ is closed under arbitrary intersections. 
Also, recall the notions of 
\emph{Maximum} and \emph{Extremal} classes \citep*[see e.g.,][]{floyd:95,lawrence:83,Bandelt06lopsided,moran2016labeled}.
A class $\H$ is a \emph{Maximum class} if, 
for every integer $m \geq \vc$ and every 
distinct $x_1,\ldots,x_m \in \X$, we have  
$|\{ (h(x_1),\ldots,h(x_m)) : h \in \H \}| = 
\sum_{i=0}^{\vc} \binom{m}{i}$ 
\citep*{floyd:95}.  
A class $\H$ is an \emph{Extremal class} if, 
for every sequence $x_1,\ldots,x_m$ shattered by $\H$, 
$x_1,\ldots,x_m$ is also shattered by set of classifiers 
in $\H$ that agree on all of $\X \setminus \{x_1,\ldots,x_m\}$.
It is known that every Maximum class is Extremal.

\begin{example}
\label{ex:maxclasses}
Let $\H$ be an intersection-closed class or an Extremal class (e.g., any Maximum class) with VC dimension $\vc$. Then
$\covcmod \leq \vc+1$.
\end{example}

In particular (due to Lemma~\ref{lem:all-the-hellys}), this means that any \emph{closed} class that is 
intersection-closed, 
Maximum, or Extremal will have $\covcmod = \covcproj = \covc \leq \vc+1$.

As mentioned, all of the inequalities in Lemma~\ref{lem:all-the-hellys} can be strict, 
due to the larger quantity being infinite while the smaller one is finite.
We now discuss this issue, and how it relates to whether a class is closed
(in terms of the definition of ``closed'' above, borrowed from \citealp*{kane19a}).
In particular, in all these examples, merely adding the limit cases into the class suffices 
to bring the complexity measures into agreement.

We begin with a simple example where $\covcmod = 2$ but $\covcproj = \infty$.

\begin{example}
\label{ex:singletons}
Consider $\X=\nats$ and $\H = \{ 2\ind_{\{t\}}-1 : t \in \X \}$ the class of \emph{singletons}.
The only finite hollow star sets are of size $2$: namely, sets $\{(x,1),(x',1)\}$ and $\{(x,1),(x,-1)\}$.
Thus, $\covcmod = 2$.  However, for any finite set $\H^\prime \subset \H$ of size at least $3$, 
${\rm Majority}(\H^\prime)$ is $-1$ everywhere on $\X$, but any $\ell < |\H^\prime|$ 
has $\X_{\H^\prime,\ell} = \X$, so that we must have $\covcproj > \ell$; thus, $\covcproj = \infty$.

We note that the constant function that is $-1$ everywhere is a pointwise limit of 
functions in $\H$.  Furthermore, by adding just this one additional function, 
the modified class $\H \cup \{x \mapsto -1\}$ has $\covc = 2$ (and hence $\covcproj = \covcmod = 2$ as well), 
and furthermore is closed (in the above sense).
\end{example}

Next we describe a simple example where $\covcproj = 2$ but $\covc = \infty$.

\begin{example}
\label{ex:thresholds}
Consider $\X = [0,\infty)$ and $\H = \{ 2 \ind_{[t,\infty)} - 1 : t \in [0,\infty) \}$
the class of \emph{threshold} functions.
For any finite $\H^\prime \subset \H$, the majority vote classifier is simply a median threshold from $\H^\prime$, 
and hence there is always an $h \in \H$ that agrees with ${\rm Majority}(\H^\prime)$ on all of $\X$.
Therefore, $\covcproj = 2$, the smallest possible value of $\covcproj$.
On the other hand, the set $S = \{ (x,-1) : x \in \X \}$ is unrealizable by $\H$, 
but there is no finite set witnessing this fact, and therefore $\covc = \infty$.

As in the previous example, we may note that the constant function that is $-1$ everywhere 
is a pointwise limit of functions in $\H$, and that by adding just this one additional 
function, the class $\H \cup \{x \mapsto -1\}$ has $\covc = 2$ (as any unrealizable set 
must contain an $(x,1)$ and $(x',-1)$ with $x \leq x'$), and hence $\covc=\covcproj=\covcmod$; 
furthermore, this modified class is also closed (in the sense described above).
\end{example}

These distinctions can occur in more extreme forms as well, as the following example illustrates.

\begin{example}
\label{ex:intervals}
Consider $\X = \reals$ and $\H = \{ 2 \ind_{[a,b]} - 1 : -\infty < a \leq b < \infty \}$
is the class of closed bounded intervals.
This class is intersection-closed (and Maximum) and has $\vc=2$.  
Furthermore, one can verify that $\covcmod = 3$.  However, 
the all-constant sets $S = \{ (x,y) : x \in \X \}$ (for either $y \in \{-1,1\}$) 
has no finite subset witnessing its non-realizability, 
and therefore $\covc = \infty$.  This is also true of the half-line sets $\{ (x,1) : x \geq a \}$ 
and $\{ (x,1) : x \leq a \}$, $a \in \reals$.

In this case, we must add an infinite set of classifiers 
to the class $\H$ if we wish to bring the complexity measures into agreement.
Specifically, if we let $\H$ be the set of \emph{all} intervals (including 
unbounded intervals, closed intervals, open intervals, half-open half-closed intervals, and the empty interval),
then we will have $\covcmod = \covcproj = \covc = 3$.
\end{example}
We note that it is possible to extend this argument to the case of rectangles on $\mathbb{R}^n$, 
where the set of limit classifiers becomes even more diverse.

\subsection{Star number and sample complexity of ERM}
\label{sec:star-ERM}

We finish this section by recalling the known sample complexity bounds for ERM and its relation with the star number.
We proceed with the following notion (which will also be used in the main results below).  
For any $x \geq 0$, define $\Log(x) = \max\{\ln(x),1\}$.
\begin{definition}[The worst-case sample complexity of ERM]
For any $\eps, \delta \in (0, 1)$, the worst-case sample complexity of ERM, denoted by $\SC_{{\rm ERM}}(\epsilon,\delta)$, is the smallest integer $n$ such that for every possible data distribution $\mathcal{P}$ and every $\target \in \H$, for $S = (X_{1:n},\target(X_{1:n}))$ with $X_{1:n} \sim \Px^n$, 
\[
\Pr\left(\sup\limits_{h \in \H[S]}\er(h,\target) \le \eps\right) \ge 1 - \delta.
\]
\end{definition}
The following bound was shown by \citep*{hanneke:16b}
\begin{equation}
\label{eq:ermsamplecomplexity}
\frac{1}{\eps}\left(\vc + \Log\!\left(\s \wedge \frac{1}{\eps}\right) + \Log\!\left(\frac{1}{\delta}\right)\right) 
\lesssim \SC_{{\rm ERM}}(\epsilon,\delta) \lesssim 
\frac{1}{\eps}\left(\vc\Log\!\left(\frac{\s}{\vc} \wedge \frac{1}{\eps}\right) + \Log\!\left(\frac{1}{\delta}\right)\right).
\end{equation}

In particular, the assumption $\s < \infty$ is a simple necessary and
sufficient condition for the existence of a distribution-free bound on the error rates of all
ERMs converging at a rate inversely proportional to the sample size. Our results below show that a much weaker assumption $\covcproj < \infty$ implies an analog of the above property but only for \emph{some} proper learners instead of \emph{all} ERMs. It is important to notice that the projection number is finite for many expressive VC classes. At the same time, the star number, while implying an upper bound for $\covcmod$, is infinite except for some relatively simple VC classes. In particular,\; $\s = \infty$ for Halfspaces in $\mathbb{R}^n$ if $n \ge 2$ and there are simple examples of intersection-closed and maximum classes for which the star number is infinite \citep*{hanneke:15b}. 

\section{Upper bounds}
\label{sec:upper-bound}
Our upper bound will be established for the following algorithm which is based on the optimal PAC learner of \citet*{hanneke:16a}.
The main modifications compared to the original 
algorithm (and analysis thereof) 
involve using $\covcproj+1$ recursive calls, rather than $3$ calls 
as in the original algorithm, 
and using the projection operator introduced above, 
to replace a majority vote classifier with an element of $\H$ 
so that the algorithm is a proper learner.
\begin{bigboxit}
$\alg(S;T)$\\
1. If $|S| < 4$, Return ${\rm ERM}(S \cup T)$\\
2. Let $S_{0}$ be the first $\lceil |S|/2 \rceil$ points in $S$\\
3. Let $S_{1},\ldots,S_{\covcproj+1}$ be independent uniform subsamples of $S \setminus S_{0}$ of size $\lfloor |S|/4 \rfloor$\\
4. Let $h_{i} = \alg(S_{0}; T \cup S_{i})$ for each $i =1,\ldots,\covcproj+1$\\
5. Return $\hat{h} = {\rm{Proj}}_{\H}(h_{1},\ldots,h_{\covcproj+1})$.
\end{bigboxit}

To be precise, in Step 3, $S_i$ is sampled 
without replacement (i.e., without duplicates).
For this algorithm, we have the following theorem, 
representing one of the main results of this article.\footnote{We implicitly assume that $\H$ 
satisfies conditions so that all of the relevant 
random variables in the analysis are measurable.
This is always the case if $\X$ is countable.
For the uncountable case, we refer the reader 
to discussions by \citet*{blumer:89,van-der-Vaart:96,van-handel:13}.}
We present its proof in Appendix~\ref{sec:upperboundsproofs}.

\begin{theorem}
\label{thm:upper-bound}
For any class $\H$, the proper sample complexity satisfies
\begin{equation*}
\SC_{{\rm prop}}(\epsilon,\delta) = 
O\!\left( \frac{\covcproj^2}{\epsilon} \left( \vc \Log(\covcproj) +  
\Log\!\left(\frac{1}{\delta}\right) \right) \right),
\end{equation*}
and this is achieved by $\Alg(S; \emptyset)$.
\end{theorem}

Interestingly, the algorithm $\alg$ above is not necessarily an ERM: that is,
it is not always consistent with the training sample.  However, it is also 
possible to define a sample-consistent variant of the algorithm, with 
only a loss of a factor of $\covcproj$ in the sample complexity bound.
Note that the standard VC bounds for analyzing ERM apply to \emph{any} 
ERM classifier, whereas the bound we establish for this learning algorithm 
only holds for this specific choice of ERM classifier, which therefore 
sometimes performs significantly better than the worst ERM. Let us proceed with the formal definition
and theorem, the proof of which is included in 
Appendix \ref{sec:upperboundsproofs}.

\begin{bigboxit}
$\alg_{\text{ERM}}(S)$\\
1. If $|S| < {\covcproj + 1}$, Return ${\rm ERM}(S)$\\
2. Split $S$ into $\covcproj+1$ disjoint subsets $S_{i}$ of size at least $\lfloor |S|/(\covcproj+1) \rfloor$ \\
3. Set $h_i = \alg_{\text{ERM}}(\bigcup_{i' \neq i} S_{i'})$ for $i = 1, \ldots, \covcproj + 1$.\\
4. Return $\hat{h} = {\rm{Proj}}_{\H}(h_{1},\ldots,h_{\covcproj+1})$.
\end{bigboxit}

To be precise, in Step 2, $S$ is split into
the subsets $S_i$ based purely on the indices: 
for instance, take $S_1$ as the first $\lfloor |S|/(\covcproj+1) \rfloor$ 
data points in the sequence,   
$S_2$ as the next $\lfloor |S|/(\covcproj+1) \rfloor$ points in the sequence, and so on,
with $S_{\covcproj+1}$ as the last 
$|S| - \covcproj \lfloor |S| / (\covcproj+1) \rfloor$ 
data points in the sequence.

\begin{theorem}
\label{thm:upper-bound-second}
For any class $\C$, the sample complexity 
$\SC_{\alg_{\text{ERM}}}$ of $\alg_{\text{ERM}}(S)$ satisfies
\begin{equation*}
\SC_{\alg_{\text{ERM}}}(\epsilon,\delta) = 
O\!\left( \frac{\covcproj^3}{\epsilon} \left( \vc \Log(\covcproj) +  
\Log\!\left(\frac{1}{\delta}\right) \right) \right).
\end{equation*}
\end{theorem}

\section{Lower Bounds}
\label{sec:lower-bound}

For the purpose of a lower bound, we will use the 
hollow star number (Definition~\ref{def:hollowstar}).
As discussed in Lemma~\ref{lem:all-the-hellys}, for many classes
we will have $\covcmod = \covcproj = \covc$, 
in which case this result indicates that the appearance 
of $\covcproj$ in Theorem~\ref{thm:upper-bound} 
is unavoidable (though it also indicates there 
may be room to improve the specific dependence on $\covcproj$ 
in the upper bound).

\begin{theorem}
\label{thm:lower-bound}
Every class with $\covcmod < \infty$ has proper sample complexity
\begin{equation*}
\SC_{{\rm prop}}(\epsilon,\delta) = \Omega\!\left( \frac{\vc}{\epsilon} + \frac{1}{\epsilon} \Log\!\left(\frac{1}{\delta}\right) + \frac{1}{\epsilon} \Log(\covcmod) \ind[ \epsilon \leq 1/\covcmod ] \right).
\end{equation*}
Also, if\; $\covcmod = \infty$ then 
\begin{equation*}
\SC_{{\rm prop}}(\epsilon,\delta) \neq o\!\left( \frac{1}{\epsilon} \Log\!\left( \frac{1}{\epsilon} \right) \right).
\end{equation*}
In other words, there exists a sequence $\epsilon_i \to 0$ with 
$\SC_{{\rm prop}}(\epsilon_i,\delta) \geq \frac{c}{\epsilon_i}\log\!\left(\frac{1}{\epsilon_i}\right)$
for a fixed numerical constant $c > 0$.
\end{theorem}

The proof, which is inspired by the arguments in 
\citep*{daniely:14}, is deferred to Appendix \ref{sec:lowerboundproofs}. 
Let us only provide some intuition behind the proof. 
Assume, for
simplicity, that we want to lower bound the sample complexity in a 
particular regime where $\eps = \frac{1}{2(\covcmod - 1)}$, 
$\covcmod \ge 2$. 
Let $S = \{(x_1,y_1),\ldots,(x_{\covcmod},y_{\covcmod})\}$ 
be a hollow star set, and $\forall i \in 1:\covcmod$ 
let $h_i \in \H$ be such that 
$\{ j : h_i(x_j) \neq y_j \} = \{ i \}$.
We set $\target = h_{i^*}$ for some $i^* \in 1:\covcmod$ 
and set $\mathcal{P}(\{x_{i^*}\}) = 0$ 
and $\mathcal{P}(\{x_i\}) = 2\eps$ for $i \neq i^*$. 
Observe that the only way for the learner to output $\hat{h} \in \H$ 
having $\er_{\Px}(\hat{h}, \target) \le \eps$ is 
for $\hat{h}$ to agree with $h_{i^*}$ on 
\emph{all} of $S$. 
However, the learner will not be able to identify the corresponding 
point $x_{i^*}$ having zero mass before it observes $x_i$ for 
\emph{every} $i \neq i^*$.
By the standard coupon collector argument 
(see Lemma \ref{lem:couponcollect} in Appendix~\ref{sec:lowerboundproofs}) due to the fact that 
there will be some copies in the training sample we will need the sample 
size of order 
$\Omega(\covcmod\Log (\covcmod)) = \Omega\!\left(\frac{1}{\eps}\Log (\covcmod)\right)$. 
The formal application of this idea (including extension to any $\epsilon \leq 1/\covcmod$) is given in Appendix~\ref{sec:lowerboundproofs}.

Our second lower bound provides stronger guarantees. However,
it is less general.
In what follows, given $\vc$ and 
$\covc$ we present a \emph{particular} class $\H$ and a space $\X$ such that 
the desired lower bound holds. The proof of this result uses the same logic and the technical details are deferred to Appendix \ref{sec:lowerboundproofs}.
\begin{theorem}
\label{thm:sometimes-tight}
There is a numerical constant $c > 0$ such that, 
for any value of $\vc \geq 1$ and $2 \leq \covc<\infty$, 
there exists a space $\X$ and a class $\H$ with VC dimension $\vc$ 
and dual Helly number $\covc$, 
for which the proper sample complexity satisfies
\begin{equation*}
\SC_{{\rm prop}}(\epsilon,\delta) \geq \frac{c}{\epsilon} \left( \vc \Log\!\left(\frac{\covc}{\vc} \land \frac{1}{\epsilon}\right) + \Log\!\left(\frac{1}{\delta}\right) \right),
\end{equation*}
for every $\epsilon \in (0,1/8)$ and $\delta \in (0,1/100)$. Furthermore, 
for any $\vc \geq 1$ there exists $\X$ and 
a space $\H$ with VC dimension $\vc$ and hollow star number $\covcmod=\infty$ and 
\begin{equation*}
\SC_{{\rm prop}}(\epsilon,\delta) \geq \frac{c}{\epsilon}\left( \vc \Log\!\left(\frac{1}{\epsilon}\right) +  \Log\!\left(\frac{1}{\delta}\right) \right).
\end{equation*}
\end{theorem}

\section{Stable compression schemes and optimality of SVM for PAC learning Halfspaces}
\label{sec:SVM}

In this section we establish a new generalization bound for a special type of compression scheme, 
referred to as a \emph{stable} compression scheme, which removes a log factor (which is known to 
not be removable for general compression schemes).  We apply the result to obtain new tighter 
bounds for several quantities of interest in the learning theory literature.

As our main application, we apply this new bound to resolve a long-standing open question: namely, showing that 
the well-known \emph{support vector machine} (SVM) learning algorithm for Halfspaces
achieves the optimal sample complexity.  This resolves a question posed by \citet*{vapnik:74}, 
which has received considerable attention in the literature \citep*[e.g.,][]{blumer1989learning, balcan:13,zhivotovskiy2017,hanneke:19b, long2020}. According to a note by A. Chervonenkis (see Chapter I in \citealp*{vovk2015measures}) the in-expectation 
version of the risk bound for SVM had been proven in 1966, even before the renowned uniform law of large numbers 
was announced \citep*{Vapnik68}.
However, obtaining a high-probability version 
of the bound, without introducing additional 
log factors, remained an open problem since then.
The following theorem resolves this problem.
Furthermore, this is also the first proof that the optimal sample complexity for Halfspaces 
is achievable by \emph{some} proper learner.

\begin{theorem}
\label{thm:svm}
The PAC sample complexity of SVM in $\mathbb{R}^{n}$ is
\[
\SC_{{\rm SVM}}(\eps,\delta) = \Theta\left(\frac{n}{\eps} + \frac{1}{\eps}\log \frac{1}{\delta}\right).
\]
\end{theorem}

The proof is presented below, after the abstract result that implies it.
This bound improves the sample complexity bound 
$
\SC_{{\rm SVM}}(\epsilon,\delta) = O\left(\frac{n}{\eps}\log \frac{1}{\delta}\right)
$
shown in \citep*{zhivotovskiy2017} and the general bound 
$
\SC_{{\rm ERM}}(\epsilon,\delta) = O\left(\frac{n}{\eps}\log \frac{1}{\eps} + \frac{1}{\eps}\log \frac{1}{\delta}\right)
$
holding for any ERM \citep*{vapnik:74,blumer:89}.

\subsection{Stable compression schemes}
\label{sec:stable-compression}

We present here our new generalization bound for stable compression schemes.
Let us recall some standard definitions. A \emph{sample compression scheme} consists of 
two functions: a \emph{compression function}
$\kappa : \bigcup_{m=0}^{\infty}(\X \times \Y)^{m} \to \bigcup_{m = 0}^{\infty}(\X \times \Y)^{m}$,
and a \emph{reconstruction function} $\rho: \bigcup_{i = 0}^{\infty}(\X \times \Y)^{i} \to \Y^{\X}$. 
The following definition is due to \citet*{littlestone:86}.

\begin{definition}[Sample compression scheme]
 The pair of functions $(\kappa,\rho)$ define a sample compression scheme of size $\ell$ if 
 for any $m \in \nats$ and $h \in \H$ and any sample $x_{1:m} \in \X^m$ 
 it holds that 
 $\kappa((x_{1:m}, h(x_{1:m})) \subseteq (x_{1:m}, h(x_{1:m}))$ 
and  $|\kappa((x_{1:m}, h(x_{1:m}))| \leq \ell$, 
 and the classifier $\hat{h} = \rho(\kappa((x_{1:m}, h(x_{1:m})))$ satisfies $\hat{h}(x_{1:m}) = h(x_{1:m})$:
 that is, it recovers $h$'s classifications for the full original $m$ examples. 
\end{definition}

\noindent We work with the following natural definition taking its roots in \citep*{vapnik:74}.

\begin{definition}[Stable compression scheme]
A sample compression scheme $(\kappa, \rho)$ is called \emph{stable} 
if for any $m \in \nats$, $h \in \H$,
$x_{1:m} \in \X^m$ and any $(x, h(x)) \in (x_{1:m}, h(x_{1:m})) \setminus \kappa((x_{1:m}, h(x_{1:m}))$ 
it holds that 
\[
\kappa((x_{1:m}, h(x_{1:m})) \setminus (x, h(x))) = \kappa((x_{1:m}, h(x_{1:m})).
\]
\end{definition}
This definition means that removing any $(x, h(x))$ \emph{not} belonging to the compression set, 
the compression set of the sub-sample remains the same. 
Finally, we say that the sample compression scheme $(\kappa, \rho)$ is \emph{proper} 
if the image of the reconstruction function $\rho$ is contained in $\H$.

It is known that 
for any stable compression scheme $(\kappa,\rho)$ of any size $\ell$, 
for any $\Px$ and $\target \in \H$, for any $m \in \nats$, 
$\E[\er(\rho(\kappa((X_{1:m},\target(X_{1:m}))))] \leq \frac{\ell}{m+1}$.
This follows from the 
leave-one-out analysis of 
\citet*{vapnik:74} (they 
argue this specifically for SVM, 
but the same argument works for any 
stable compression scheme; see also \citealp*{haussler:94, zhivotovskiy2017}).

The best known PAC generalization bound on 
$\er(\rho(\kappa(X_{1:m},\target(X_{1:m}))))$ (i.e., holding with probability $1-\delta$) 
valid for \emph{any} sample compression scheme of a size $\ell$ 
is due to \citet*{littlestone:86} (see also \citealp*{floyd:95}):
$O\!\left( \frac{1}{m} \left( \ell \log(m) + \log\!\left(\frac{1}{\delta}\right) \right) \right)$, 
where the $\log(m)$ factor improves to $\Log\!\left(\frac{m}{\ell}\right)$ if $\rho$ is permutation-invariant. \citet*{floyd:95} showed that there exist spaces $\H$ and compression schemes for $\H$  
for which this log factor cannot be improved.  However, in the special case of \emph{stable} 
compression schemes, \citet*{zhivotovskiy2017} established a bound 
$O\!\left(\frac{\ell}{m}\log\!\left(\frac{1}{\delta}\right)\right)$, 
which is sometimes better.

As one of the main contributions 
of this work, the following result improves this PAC generalization bound
by completely removing the log factor from the bound of \citet*{littlestone:86}.
A simple proof of this result is provided 
in Appendix~\ref{app:samplecomp}.

\begin{theorem}
\label{thm:samplecomp}
Assume that $\H$ has a stable sample compression scheme $(\kappa,\rho)$ of size $\ell$.
Then, for any $\Px$ and any $\target \in \H$, for any integer $m > 2\ell$, 
given an i.i.d.\ sample $S = (X_{1:m},\target(X_{1:m}))$ of size $m$, 
for any $\delta \in (0,1)$, 
we have with probability at least $1 - \delta$,
\[
\er(\rho(\kappa(S))) < \frac{2}{m - 2\ell} \left( \ell \ln(4) + \ln\!\left(\frac{1}{\delta}\right) \right). 
\]
\end{theorem}

This result immediately yields the following proof of Theorem~\ref{thm:svm}
establishing optimality of SVM for PAC learning Halfspaces.

\begin{proof}[of Theorem~\ref{thm:svm}]
Since SVM may be expressed as a stable compression scheme of size $n + 1$ 
(see \citep*{long2020} for a transparent proof of this fact, originally proven 
by \citealt*{vapnik:74}), 
the upper bound is immediate from Theorem~\ref{thm:samplecomp}. 
The lower bound follows from \citep*{ehrenfeucht:89}.
\end{proof}

\paragraph{Maximum classes.} 
As a second application of Theorem~\ref{thm:samplecomp} 
to proper learning, consider any 
\emph{maximum class} $\H$ (\citealp*{floyd:95}; 
recall the definition from Section~\ref{subsec:examples}).
Every maximum class $\H$ is known to have a proper stable compression scheme of size $\vc$ \citep*{chalopin:19};
this follows from Theorem 5.1 and condition (R2) of Theorem 6.1 from the paper of \citealp*{chalopin:19}
(specifically, since every realizable data set has exactly one concept in $\H$ consistent with the data, whose compression 
set is contained in the data set, removing any sample not in that compression set cannot change 
the identity of this unique concept).
Hence, we have the following
corollary of Theorem~\ref{thm:samplecomp}, establishing 
(for the first time) 
that maximum classes are 
properly learnable with sample complexity 
of the same order as the optimal PAC sample complexity.

\begin{corollary}
\label{cor:proper-extremal}
For $\H$ which is a Maximum class, 
\begin{equation*}
\SC_{{\rm prop}}(\eps,\delta) = \Theta\!\left( \frac{\vc}{\eps}+\frac{1}{\eps}\log\!\left(\frac{1}{\delta}\right) \right).
\end{equation*}
\end{corollary}

\begin{remark}
We note that Theorem~\ref{thm:samplecomp} 
is also able to recover the optimal PAC 
sample complexity for the \emph{Closure} algorithm 
for \emph{intersection-closed} classes 
\citep*{helmbold:90}; this sample 
complexity result was already 
known via different arguments 
\citep*{darnstadt:15,hanneke:16b,auer:07}, 
though Theorem~\ref{thm:samplecomp} offers 
improved numerical constants.
\end{remark}

\begin{remark}
A further implication of Theorem~\ref{thm:samplecomp} and 
Theorem~\ref{thm:lower-bound} together is that
any class $\H$ with $\covcmod = \infty$ does \emph{not} have a 
proper stable compression scheme of bounded size.
\end{remark}

\section{Remaining Gaps and Open Questions}
\label{sec:gaps}
Although, our upper and lower bounds give a description of proper PAC sample complexity 
in many cases, there are still some situations not explained by our general analysis.
Consider, for example, the case $\X=\nats$ and $\H = \{ 2\ind_{\{t\}}-1 : t \in \X \}$ 
the class of singletons (Example~\ref{ex:singletons}). Since $\vc = 1$ and $\s = \infty$
we have 
$\SC_{{\rm ERM}}(\eps, \delta) = \Theta\left(\frac{1}{\eps} \!\left(\log \frac{1}{\eps} + \log \frac{1}{\delta}\right)\right)$.
As discussed in Example~\ref{ex:singletons}, we have $\covcmod = 2$ 
but $\covcproj = \infty$.  Thus, our general upper bound in 
Theorem~\ref{thm:upper-bound} does not match the lower bound implied 
by Theorem~\ref{thm:lower-bound}, and also does not match 
the optimal (improper) sample complexity 
$\SC(\eps, \delta) = \Theta \left(\frac{1}{\eps}\log \frac{1}{\delta}\right)$.
However, it turns out there does exist a proper stable compression scheme 
of size $1$ for this $\H$, and therefore Theorem~\ref{thm:samplecomp} implies that 
$\SC_{{\rm prop}}(\eps,\delta) = \Theta\left(\frac{1}{\eps}\log\frac{1}{\delta}\right)$, 
matching the optimal (improper) sample complexity.  Specifically, 
for any data set $S$, if $S$ contains a positive example $(x,1)$, we can 
define $\kappa(S) = \{(x,1)\}$, which is clearly stable, and from which 
we can clearly reconstruct the labeling of $S$.  On the other hand, if $S$ 
is all negative examples, we can choose $\kappa(S)$ as $\{(x,-1)\}$ 
for the \emph{largest} $x$ in the set $S$, and reconstruct 
$\rho(\{(x,-1)\}) = 2 \ind_{\{x+1\}} - 1$, which is clearly consistent with $S$;
since $x$ is largest, removing any other point from $S$ does not affect 
the choice of compression set.
It is therefore important to ask whether perhaps, for \emph{every} class $\H$,
the optimal form of 
$\SC_{{\rm prop}}(\eps,\delta)$ is characterized by $\covcmod$ 
(in addition to $\vc$, $\eps$, $\delta$).
It is also interesting to consider whether the optimal size of a 
stable proper compression scheme is also always characterized by $\covcmod$ 
(in addition to $\vc$).
Another interesting direction is to sharpen the dependence on $\covcproj$ in our upper bounds.

Finally, we remark that the result of Section \ref{sec:SVM} is closely related to the question of obtaining high probability generalization bounds for various learning algorithms that are stable with respect to small perturbations in the learning sample. In the context of uniformly stable algorithms the recent results \citep*{feldman19a, bousquet2019sharper} provide sharp high probability bounds and the proofs are based on a sub-sampling argument: the learning algorithm is tested on some carefully chosen parts of the learning sample. As in the case of uniformly stable algorithms, stable compression schemes are known to easily provide sharp in-expectation risk bounds \citep*{haussler:94}. The challenging part, already pointed out by \cite{vapnik:74}, is to prove that these algorithms admit sharp high probability upper bounds. Our results, also based on related arguments, are the first to prove that the optimal high probability sample complexity bound is possible for stable compression schemes, including SVM.
\bibliography{learning}

\appendix
\section{Omitted proofs of Section \ref{sec:definitions}}
\label{sec:omittedexamples}
\begin{proof}[of Lemma \ref{lem:all-the-hellys}]
We note that many of these arguments were given in some form 
in the course of the proofs of \citet*{kane19a}.  However, we 
provide the proofs explicitly here, particularly since our 
context is slightly different.

First, on a technical note, we remark that because of our assumption $|\H| \geq 3$ stated 
initially, all of $\covcmod$,$\covcproj$,$\covc$ are at least $2$ (so that the restriction to $\covcproj \geq 2$ 
in its definition does not affect the claims).

For the first claimed inequalities, 
given the $\covcmod$ classifiers $\H^\prime$ in $\H$ witnessing a hollow star,
for any $\ell < \covcmod$ the region $\X_{\H^\prime,\ell}$ contains the hollow star, 
and hence the majority vote of the $\H^\prime$ classifiers is unrealizable on 
$\X_{\H^\prime,\ell}$.  Therefore, $\covcproj \geq \covcmod$.

If $\covc = \infty$, then trivially $\covcproj \leq \covc$, so suppose $\covc < \infty$.
Suppose some finite 
multiset $\H^\prime \subseteq \H$ has no $h \in \H$ that coincides with ${\rm Majority}(\H^\prime)$ on $\X_{\H^\prime,\covc}$.
Then $S = \{ (x,{\rm Majority}(\H^\prime)(x)) : x \in \X_{\H^\prime,\covc} \}$ 
is an unrealizable set.  Therefore, it contains a subset $W$ of size at most $\covc$ 
that is also unrealizable.
But (by definition of $\X_{\H^\prime,\covc}$) each point $(x,y)$ in $W$ contradicts strictly fewer than $|\H^\prime|/\covc$ 
elements in $\H^\prime$, so that there must be at least one $h \in \H^\prime$ that survives: a contradiction.
Therefore, $\covcproj \leq \covc$.

It remains to show that these quantities are all equal when $\covc < \infty$.
Let $S$ be an unrealizable set 
such that the smallest unrealizable subset $W$ has size $\covc < \infty$.
If $W$ is not a hollow star set, then there exists a point $(x,y) \in W$ 
such that $(W \setminus \{(x,y)\}) \cup \{(x,-y)\}$ is also not realizable, 
which implies $W \setminus \{(x,y)\}$ is also not realizable: a contradiction.
Therefore, $W$ is a hollow star (indeed, any unreducible unrealizable set is a hollow star), 
and hence $\covc \leq \covcmod$.  The equalities then follow from the first claim, established above.

For the remaining claim, if $\H$ is closed and yet $\covc = \infty$, 
it implies there is a sequence $S_i$ of sets with $\H[S_i]=\emptyset$
for which the smallest $W_i \subseteq S_i$ with $\H[W_i]=\emptyset$ 
has $\lim_{i \to \infty} |W_i| = \infty$, yet each $W_i$ is finite (due to the ``closedness'' assumption).
As above, these minimum-size $W_i$ sets must be hollow star sets, 
which implies there is no finite bound on the size of all finite hollow star sets.
Therefore $\covcmod = \infty$, and the inequality established above 
then implies $\covcproj=\infty$ as well.
\end{proof}
\begin{proof}[for Example \ref{ex:dualhellyforhalfspaces}]
This follows from Proposition 2.8 in \citep*{braverman2019convex}.
Indeed, this implies that $\covcmod \leq n+2 $ since if
a finite unrealizable sample $S$ 
is a hollow star set, 
then in particular every proper subsample of $S$ must be realizable;
however, by Proposition 2.8 in \citep*{braverman2019convex}  
the set $S$ must contain an unrealizable subsample of size at most $n+2$, and hence it must be that $\lvert S\rvert \leq n+2$.
To see that $\covcmod \geq n+2$, pick $x_1,\ldots, x_{n+1}\in\mathbb{R}^n$ to be the vertices of a simplex, and choose 
\[S=\left\{(x_1,+1),\ldots, (x_{n+1}, +1), \Bigl(\frac{x_1 + \ldots + x_{n+1}}{n+1},-1\Bigr)\right\}.\]
We leave it to the reader to verify that the above $S$ witnesses that $\covcmod\geq \lvert S\rvert=n+2$.

It remains to show that $\covcproj=n+2$. By Lemma~\ref{lem:all-the-hellys}
    it suffices to show that $\covcproj\leq n+2$.
    Let $\H^{\prime}$ be a finite collection of halfspaces in $\mathbb{R}^n$.
    We need to show that there exists a halfspace $h$ which agrees with
    $\rm{Majority}(\H^{\prime})$ on the set $\mathcal{X}_{\H^{\prime}, n+2}$.
    Let $\mathcal{X}_+ \subseteq \mathcal{X}_{\H^{\prime}, n+2}$ denote the set of all points $x \in \mathcal{X}_{\H^{\prime}, n+2}$
    such that $\rm{Majority}(\H^{\prime})(x)=+1$ and similarly let
    $\mathcal{X}_- \subseteq \mathcal{X}_{\H^{\prime}, n+2}$ denote the set of all points $x \in \mathcal{X}_{\H^{\prime}, n+2}$
    such that $\rm{Majority}(\H^{\prime})(x)=-1$.
    We first claim that the convex hulls 
    $\mathsf{conv}(\mathcal{X}_+)$ and $\mathsf{conv}(\mathcal{X}_{-})$
    are disjoint; 
    indeed, otherwise by Proposition 2.8 in \citep*{braverman2019convex} 
    there exist $S_+\subseteq \mathcal{X}_+, S_-\subseteq \mathcal{X}_-$ 
    such that $\mathsf{conv}(S_-)\cap \mathsf{conv}(S_+) \neq\emptyset$ and $\lvert S_-\rvert + \lvert S_+\rvert\leq n+2$
    However, every $x\in S_-\cup S_+$ is classified correctly by more than $1-\frac{1}{n+2}$ fraction
    of the halfspaces in $\H^{\prime}$ and hence there must be a halfspace in $\H^{\prime}$ that classifies correctly
    all points in $S_+\cup S_-$ and so $\mathsf{conv}(S_-)\cap \mathsf{conv}(S_+) =\emptyset$,
    which is a contradiction.
    
Having established that $\mathsf{conv}(S_-)\cap \mathsf{conv}(S_+) =\emptyset$, we are ready to finish the proof.
    Indeed, by the {\it Hyperplane Separation Theorem} there exists a linear function $L:\mathbb{R}^n\to \mathbb{R}$
    and a value $v$ such that $L(x)\leq v$ for every $v\in \mathcal{X}_-$ and $L(x)\geq v$ for every $v\in \mathcal{X}_+$.
    However, note that in fact $L(x) < v$ for every $x\in\mathcal{X}_-$: this follows because $\mathcal{X}_-$
    is an open set (indeed, it can be written as a union of (finite) intersections of sets of the form $h^{-1}(-1)$,
    where $h\in\H^{\prime}$, each of which is an open set). 
    Thus, the halfspace $\{x : L(x)\geq v\}$ 
    agrees with $\rm{Majority}(\H^{\prime})$ on the set $\mathcal{X}_{\H^{\prime}, n+2}$, as required.\footnote{Note that we use the 
    definition of Halfspaces where the positive side of each halfspace is closed: 
    i.e.,\ every halfspace is of the form $\sign(L(x)-v)$, where $\sign(0) = +1$.}
\end{proof}
\begin{proof}[for Example \ref{ex:maxclasses}]
Let us sketch the proof. The bound ${\covcmod} \leq \vc+1$ on the dual Helly number for extremal classes 
follows from the fact that: (i) the complement of an extremal class is extremal,
and (ii) the \emph{one-inclusion graph} on an extremal graph projected to any data set is connected 
(we refer the reader to \citealp*{Bandelt06lopsided} and \citealp*{moran2016labeled} for these definitions and results).
In particular, it follows that the one-inclusion graph of the 
complement of $\H$ projected to any data set is connected.

Since a hollow star  
corresponds to a one-inclusion graph 
where the star's center is an isolated vertex in the 
one-inclusion graph of the complement of $\H$ projected to the points, 
it must be the only vertex in the complement, 
which means any strict subset of the points is shattered by $\H$.
This immediately implies $\covcmod - 1 \leq \vc$.

The case of intersection-closed classes is also straightforward.
Let $S = \{(x_1,y_1),\ldots,(x_{k},y_{k})\}$ be a 
finite hollow star set, and let $h_1,\ldots,h_k$ be 
elements of $\H$ such that  
$\{ j : h_i(x_j) \neq y_j \} = \{i\}$
for each $i$.
Denote by $m$ the number of 
$y_i$ values equal $-1$. We will first show that 
$m \le 1$. Indeed, if the are at least two values 
$y_i$ and $y_j$ ($i \neq j$) both equal to $-1$, 
then letting 
$h_0$ be the classifier in $\H$ with 
$\{ x : h_0(x) = 1 \} = \{ x : h_i(x) = 1 \} \cap \{ x : h_j(x) = 1 \}$ (which exists since $\H$ is intersection-closed), we would have $h_0$ 
correct on all of $S$: a contradiction to $S$ 
being unrealizable.  Next we argue that 
there are at most $\vc$ values $y_i$ equal $1$: 
that is, $k-m \leq \vc$.  Suppose 
$y_{i_1},\ldots,y_{i_{k-m}}$ are equal $1$. 
Then for any $y_{i_{1}}^{\prime},\ldots,y_{i_{k-m}}^{\prime} \in \Y$, there exists a classifier $h \in \H$ 
with $\{ x : h(x) = 1 \} = \bigcap \{ \{ x : h_{i_j}(x) = 1 \} : y_{i_j}^{\prime} = -1 \}$, and this $h$ has  
$(h(x_{i_1}),\ldots,h(x_{i_{k-m}})) = (y_{i_1}^{\prime},\ldots,y_{i_{k-m}}^{\prime})$; 
thus, the set $\{x_{i_1},\ldots,x_{i_{k-m}}\}$ 
is shattered by $\H$, and hence has size at most $\vc$.
Altogether, we have that $k = (k-m) + m \leq \vc + 1$.
Since this applies to \emph{any} finite hollow star set $S$, 
we conclude that $\covcmod \leq \vc + 1$.
\end{proof}
\begin{proof}[for Example \ref{ex:intervals}]
We provide the details for the final claim that 
the augmented class has $\covc = \covcproj = \covcmod = 3$.
Since the class 
is intersection-closed with $\vc=2$, 
Example~\ref{ex:maxclasses} implies $\covcmod \leq 3$.
Thus, since $\covcmod \geq 3$ 
(witnessed by the hollow star set 
$\{(1,1),(2,-1),(3,1)\}$),  
Lemma~\ref{lem:all-the-hellys} 
implies it suffices to show the class is closed.
Let $S$ be an infinite unrealizable set; 
we aim to show it must contain a finite 
unrealizable subset. 
If $S$ contains $(x,1)$ and $(x,-1)$ for some $x$, 
then $\{(x,1),(x,-1)\}$ is a finite unrealizable 
subset.  
Otherwise, suppose no such $x$ exists. 
Notice that 
$S$ must contain at least one point having label $1$, 
since otherwise $S$ would be realizable by  
the constant classifier $h_{-1} = -1$, 
which is in the class.
Let $\underline{x} = \inf\{ x : (x,1) \in S \}$ 
and $\bar{x} = \sup\{ x : (x,1) \in S \}$.
Note that every $(x,y) \in S$ with 
$x > \bar{x}$ has $y = -1$ and similarly 
every $(x,y) \in S$ with $x < \underline{x}$ 
has $y = -1$.  In particular, this implies that 
if every $(x,y) \in S$ 
with $\underline{x} < x < \bar{x}$ has $y = 1$, 
then $S$ would be realizable by a classifier
corresponding to one of the 
intervals  
$(\underline{x},\bar{x}]$ 
(if $(\underline{x},1) \notin S$ and $(\bar{x},1) \in S$), $[\underline{x},\bar{x})$ 
(if $(\underline{x},1) \in S$ and $(\bar{x},1) \notin S$), $[\underline{x},\bar{x}]$ (if $(\underline{x},1) \in S$ and $(\bar{x},1) \in S$),
or $(\underline{x},\bar{x})$ (if $(\underline{x},1) \notin S$ and $(\bar{x},1) \notin S$): 
a contradiction.
Therefore, there exists $(x_{2},-1) \in S$ 
with $\underline{x} < x_{2} < \bar{x}$.
By the definitions of $\underline{x}$ and $\bar{x}$, 
this implies there exist finite $x_1,x_3$ 
with $x_1 < x_2 < x_3$ such that $(x_1,1)$ and $(x_3,1)$ 
are both in $S$.  Then we have that the set 
$\{(x_1,1),(x_2,-1),(x_3,1)\} \subset S$ is a 
finite unrealizable subset.
Since this applies to \emph{any} unrealizable 
infinite set $S$, we conclude that the class is closed.
\end{proof}
\section{Proofs of the upper bounds}
\label{sec:upperboundsproofs}
\begin{proof}[of Theorem \ref{thm:upper-bound}]
Fix any target concept $\target \in \C$ and 
distribution $\Px$.
We first argue that, for any finite data sets $S$ and $T$
with labels consistent with $\target$, 
the proper learner $\alg(S;T)$ 
outputs $\hat{h} \in \H$ such that 
$\hat{h}$ is correct on $T$.
Note that this is trivially true if $|S| < 4$, 
since Step 1 returns $\text{ERM}(S \cup T)$, 
which is correct on $T$ by definition.
Now, for induction, suppose $S$ is a  
correctly labeled finite data set such that, 
for any 
strict subset $S' \subset S$, and any 
correctly labeled finite data set $T'$, 
$\alg(S';T')$ returns a classifier in $\H$ 
that is correct on $T'$.
Now we extend this property to the full set $S$. 
Fix any correctly labeled finite data set $T$.
Recalling the notation from the algorithm, 
define 
$h_{{\rm maj}}(x) = {\rm Majority}(h_1(x),\ldots,h_{\covcproj+1}(x))$ 
(breaking ties to favor label $-1$, say), 
and recalling the notation \eqref{eq:majoritylabel} 
let  
\[
\X_{0} = \mathcal{X}_{\{h_1,\ldots,h_{\covcproj+1}\}, \covcproj}.
\]
By definition (from Step 5 in $\alg(S;T)$),
the classifier 
$\hat{h} = \alg(S;T) \in \H$ has 
$\hat{h}(x) = h_{{\rm maj}}(x)$
on every $x \in \X_{0}$.
Furthermore,
since $h_{i} = \alg(S_{0};T\cup S_{i})$ for each $i$, 
where $S_{0} \subset S$, 
the inductive hypothesis implies $h_{i}$ is 
correct on $T$.
Therefore, all $h_{i}$ agree on the labels in $T$, 
and hence the set of points in~$T$ is contained in $\X_{0}$, 
which implies  
$\hat{h}$ is correct on $T$ as well.
By induction, this implies that for any correctly labeled 
finite data sets $S$ and $T$, 
$\hat{h} = \alg(S;T)$ is correct on $T$.

Next we argue that, for any $m_{0} \in \nats$
and any $\delta_{0} \in (0,1)$,
if $T$ is any correctly labeled 
finite data set and $S$ is an i.i.d.\ labeled data set 
of size $m_0$, 
with $\Px$ marginal distribution and $\target$ labels, 
then with probability at least $1-\delta_{0}$,
we have that $\hat{h}$ satisfies 
\begin{equation}
\label{eqn:the-claim}
\er(\hat{h}) \leq \frac{c\cdot \covcproj^2}{m_{0}}\!\left( \vc\Log(\covcproj) + \Log\!\left(\frac{1}{\delta_{0}}\right) \right),
\end{equation}
where $c \geq 1$ is an appropriate finite 
numerical constant.  Note that this would 
imply Theorem~\ref{thm:upper-bound}, 
since setting $\delta_{0} = \delta$ 
and $m_{0}$ of size proportional to the 
claimed bound on $\SC_{\text{prop}}(\eps,\delta)$ 
from Theorem~\ref{thm:upper-bound}, 
the bound \eqref{eqn:the-claim} is less 
than $\eps$.

If $m_0 < 4$, 
the claim trivially holds, as the bound is 
greater than $1$.
In particular, 
this will be our base case in an inductive argument.
Now, for induction, suppose 
$m \geq 4$,  
and that for any $\delta_{0} \in (0,1)$ and 
any $m_{0} < m$,
if $S$ is an i.i.d.\ data set of size $m_{0}$ 
(with $\Px$ marginal distribution and $\target$ labels) 
and $T$ is any finite data set with $\target$ labels, 
then with probability at least $1-\delta_{0}$
the inequality \eqref{eqn:the-claim} 
holds for $\hat{h} = \alg(S;T)$.

Next we extend this claim to hold for $m_0 = m$.
Fix any $\delta_0 \in (0,1)$.
If 
$m < 160 \Log\!\left(\frac{6 \covcproj^2}{\delta_0}\right)$, the inequality trivially holds since the bound is
greater than $1$ (for sufficiently 
large choice of $c$), so suppose 
$m \geq 160 \Log\!\left(\frac{6 \covcproj^2}{\delta_0}\right)$.
Consider the sets $S_i$ and classifiers $h_i$ 
as defined in the specification of the 
algorithm $\alg(S;T)$ above Theorem~\ref{thm:upper-bound}, 
and with a slight abuse of notation we also use $S_i$ to denote 
the \emph{unlabeled} portion of the set $S_i$ (i.e., the points $x$ such that $(x,\target(x)) \in S_i$).
As argued above, 
each $h_i$ is correct on $T \cup S_{i}$.

For any $h$, define 
$\ER(h) = \{ x : h(x) \neq \target(x) \}$.
We claim that 
\begin{equation}
\label{eqn:ER-union-1}
\ER(\hat{h}) \subseteq \bigcup_{i,j : i \neq j} \ER(h_i) \cap \ER(h_j).
\end{equation}
To see this, note that since $\hat{h}$ agrees with $h_{{\rm maj}}$ 
on $\X_{0}$, we have 
\begin{equation}
\label{eqn:hat-ER}
\ER(\hat{h}) \subseteq 
(\X \setminus \X_{0}) \cup ( \X_{0} \cap \ER(h_{{\rm maj}} ) ).
\end{equation}
Furthermore, for any $x \in \X \setminus \X_{0}$, 
at least two values of $i$ have $h_i(x)$
different from the majority of the values  
$h_1(x),\ldots,h_{\covcproj+1}(x)$,
which means there are at least two classifiers 
predicting each label in $\Y$, 
and hence there are at least two classifiers 
$h_i$ with 
$h_i(x) \neq \target(x)$.
Furthermore, any $x$ with 
$h_{{\rm maj}}(x) \neq \target(x)$ 
certainly also has at least two $h_i$ classifiers 
with $h_i(x) \neq \target(x)$.
Therefore, the set on the right hand side 
of \eqref{eqn:hat-ER} is contained within
$\bigcup_{i,j : i \neq j} \ER(h_i) \cap \ER(h_j)$, 
and \eqref{eqn:ER-union-1} follows.

In particular, \eqref{eqn:ER-union-1} implies 
\begin{equation}
\label{eqn:hat-er-union-bound-1}
\er(\hat{h}) = \Px(\ER(\hat{h})) 
\leq \Px\!\left( \bigcup_{i,j : i < j} \ER(h_i) \cap \ER(h_j) \right) 
\leq \sum_{i,j : i < j} \Px\!\left( \ER(h_i) \cap \ER(h_j) \right).
\end{equation}
The remainder of the proof will establish that each term $\Px(\ER(h_i) \cap \ER(h_j))$ is small 
with high probability.
This is achieved using a ``Win-Win'' argument showing that
for every distinct $i, j$, either $\Px(\ER(h_i)) = \er(h_i)$ is small, or 
else $\Px(\ER(h_j) \vert \ER(h_i))$ is small.
In either case, it will follow that $\Px(\ER(h_i)\cap \ER(h_j))$ is small.

Specifically, we follow a ``conditioning'' argument.\footnote{Conditioning 
arguments of this type originate in the work of \citet*{hanneke:thesis} 
on ERM bounds and active learning, 
and were later used to analyze several different 
learning algorithms \citep*[e.g.,][]{darnstadt:15,hanneke:16b,zhivotovskiy2018localization}.
Notably, the argument was used by \citet*{simon:15}
to analyze majority votes of independent ERMs, 
and was used by \citet*{hanneke:16a} in the proof 
of the optimal PAC sample complexity. 
Its use in the present proof most closely follows 
this latter work.}
We claim that, with probability at least 
$1 - \delta_{0}/3$,
for every pair $i,j$ with $i < j$,  
either 
\begin{equation}
\label{eqn:eri-small}
\er(h_i) < \frac{320}{m}\ln\!\left(\frac{6\covcproj^2}{\delta_0}\right)
\end{equation}
or else
\begin{equation}
\label{eqn:ERi-intersection}
|\ER(h_i) \cap (S_{j} \setminus S_{i})| \geq 
\er(h_i) m / 80,
\end{equation}
where here the notation $S_j \setminus S_i$ denotes the set of samples from $S$ 
that are in $S_j$ but not $S_i$ (distinguished by their \emph{indices}, 
so that if $S$ contains 
two copies of some $x \in \X$ and one is in $S_i$ and the other in $S_j$, 
then the latter will still appear in $S_j \setminus S_i$).

Toward establishing the above claim, note that for each distinct $i,j$ we have  
\begin{align}
& \P\!\left( |\ER(h_i) \cap (S_{j} \setminus S_{i})| < \er(h_i) m / 80 
\text{ and } \er(h_i) \geq \frac{320}{m}\ln\!\left(\frac{6\covcproj^2}{\delta_0}\right) \right) \notag
\\ & \leq \P\!\left( |\ER(h_i) \!\cap\! (S_{j} \!\setminus\! S_{i})| \!<\! (1/2) \er(h_i) |S_j \!\setminus\! S_i| 
\text{ and } \er(h_i) \!\geq\! \frac{320}{m}\ln\!\left(\frac{6\covcproj^2}{\delta_0}\right) \!\text{ and } |S_j \!\setminus\! S_i| \!\geq\! \frac{m}{40} \right) \notag  
\\ & {\hskip 2cm}+ \P\!\left( |S_j \setminus S_i| < \frac{m}{40} \right). \label{eqn:ERi-eri-SjSi}
\end{align}
We begin with bounding the second term.
Note that 
$\E\!\Big[ |S_j \setminus S_i| \Big| S_i \Big] = \left( 1 - \frac{\lfloor m/4 \rfloor}{\lfloor m/2 \rfloor} \right) \lfloor m/4 \rfloor \geq m/20$ 
(noting that, since $m \geq 20$, we have $\lfloor m/4 \rfloor \geq m/5$ and $\lfloor m/2 \rfloor \geq m/3$).
Therefore, by a multiplicative Chernoff bound\footnote{As proven by \citealp*{hoeffding:63}, 
the moment generating function for sampling without replacement 
is upper bounded by the moment generating 
function for sampling with replacement, and hence the usual 
Chernoff bounds for sampling with replacement also hold 
for sampling without replacement.} 
conditioned on $S_i$, 
and the law of total probability,  
we have 
$\P( |S_j \setminus S_i| < m/40 ) \leq e^{-m/160} \leq \frac{\delta_{0}}{6 \covcproj^2}$.

Next we bound the first term in \eqref{eqn:ERi-eri-SjSi}.
We note that, conditioned on $|S_j \setminus S_i|$, 
the samples in $S_j \setminus S_i$ are i.i.d.\ (with distribution $\Px$) 
and independent of $h_i$.
Thus, a multiplicative Chernoff bound implies (almost surely)
\begin{equation*}
\P\!\Big( |\ER(h_i) \cap (S_j \setminus S_i)| < (1/2)\er(h_i)|S_j \setminus S_i| \Big| h_i, |S_j \setminus S_i| \Big)
\leq e^{-(1/8) \er(h_i) |S_j \setminus S_i|}.
\end{equation*}
Therefore, the first term in \eqref{eqn:ERi-eri-SjSi} is bounded by 
\begin{equation*}
\E\!\left[ e^{-(1/8) \er(h_i) |S_j \setminus S_i|} \ind\!\left[\er(h_i) \geq \frac{320}{m}\ln\!\left(\frac{6\covcproj^2}{\delta_0}\right) \text{ and } |S_j \setminus S_i| \geq m/40 \right] \right]
\leq \frac{\delta_0}{6 \covcproj^2}.
\end{equation*}
Altogether, we have that \eqref{eqn:ERi-eri-SjSi} is at most $\frac{\delta_0}{3 \covcproj^2}$.
Finally, by a union bound over all pairs $i,j$ with $i < j$, 
we conclude that with probability at least $1-\frac{\delta_0}{3}$, 
for every $i,j$ with $i < j$, at least one of \eqref{eqn:eri-small} or \eqref{eqn:ERi-intersection} holds, 
as claimed.

Since $h_j$ is correct on $S_j$, 
it is certainly correct on $\ER(h_i) \cap (S_j \setminus S_i)$.
Also note that the samples in 
$\ER(h_i) \cap (S_j \setminus S_i)$ are conditionally i.i.d.\ given 
$h_i$ and $|\ER(h_i) \cap (S_j \setminus S_i)|$, 
with conditional distribution $\Px(\cdot | \ER(h_i))$.
We can therefore apply the classic PAC bound for ERM  
\citep*{vapnik:74,blumer:89}, 
under the conditional distribution given $h_i$ and $|\ER(h_i) \cap (S_j \setminus S_i)|$, 
together with the law of total probability, 
to obtain that, 
for any distinct $i,j$, 
with probability at least $1 - \frac{\delta_{0}}{3 \covcproj^2}$,
\begin{equation*}
\Px( \ER(h_j) | \ER(h_i) ) 
\leq \frac{2 / \ln(2)}{| \ER(h_i) \cap (S_j \!\setminus\! S_i)|} 
\left( \vc \Log\!\left( \frac{2 e | \ER(h_i) \cap (S_j \!\setminus\! S_i) |}{\vc} \right) + \Log\!\left(\frac{6 \covcproj^2}{\delta_{0}} \right) \right).
\end{equation*}
(interpreting the bound to be infinite in the case $|\ER(h_i) \cap (S_j \setminus S_i)|=0$).
By the union bound, this holds simultaneously for all $i,j$ with $i < j$ 
with probability at least $1 - \frac{\delta_{0}}{3}$.
Combining this with the event established above, by the union bound and monotonicity of 
$x \mapsto (1/x) \Log(a x)$, 
with probability at least $1 - \frac{2}{3} \delta_{0}$, 
every $i,j$ with $i < j$ either satisfy \eqref{eqn:eri-small} 
or 
\begin{equation}
\label{eqn:ERj-ERi-conditional-bound-1}
\Px( \ER(h_j) | \ER(h_i) ) \leq \frac{160/\ln(2)}{\er(h_i)m}\left(\vc \Log\!\left(\frac{\er(h_i) m e}{40 \vc}\right) + \Log\!\left(\frac{6 \covcproj^2}{\delta_0}\right)\right).
\end{equation}

Since $S_{0}$ is i.i.d.\ (with marginal distribution $\Px$ and $\target$ labels) 
with $m/2 \leq |S_{0}| < m$, the inductive hypothesis and the union bound imply that, 
with probability at least 
$1 - \frac{\delta_{0}}{3}$, every $h_i$ has
\begin{equation*}
\er(h_i) \leq 
\frac{ c \covcproj^2}{m/2}\!\left( \vc \Log(\covcproj) + \Log\!\left(\frac{3(\covcproj+1)}{\delta_{0}}\right) \right).
\end{equation*}
Plugging this into the log in \eqref{eqn:ERj-ERi-conditional-bound-1}, 
by the union bound we have that, with probability at least $1-\delta_{0}$, 
every $i,j$ with $i < j$ either satisfy \eqref{eqn:eri-small} 
or $\Px( \ER(h_j) | \ER(h_i) )$ is upper bounded by 
\begin{equation*}
\frac{160/\ln(2)}{\er(h_i)m}\left(\vc \Log\!\left(c \covcproj^2 \left( \Log(\covcproj) + \frac{1}{\vc}\Log\!\left(\frac{3 (\covcproj+1)}{\delta_{0}}\right) \right)\right) + \Log\!\left(\frac{6\covcproj^2}{\delta_0}\right)\right).
\end{equation*}
In either case (i.e., whether \eqref{eqn:eri-small} holds or not), on this event 
\begin{align*}
& \Px( \ER(h_i) \cap \ER(h_j) ) 
= \er(h_i) \Px(\ER(h_j) | \ER(h_i))
\\ & \leq \frac{320}{m}\left(\vc \Log\!\left(c \covcproj^2 \left( \Log(\covcproj) + \frac{1}{\vc}\Log\!\left(\frac{3 (\covcproj+1)}{\delta_{0}}\right) \right)\right) + \Log\!\left(\frac{6\covcproj^2}{\delta_0}\right)\right).
\end{align*}

Combining this with \eqref{eqn:hat-er-union-bound-1} we have that, 
with probability at least $1-\delta_{0}$, 
\begin{equation*}
\er(\hat{h}) \leq \binom{\covcproj+1}{2}\frac{320}{m}\left(\vc \Log\!\left(c \covcproj^2 \left( \Log(\covcproj) + \frac{1}{\vc}\Log\!\left(\frac{3(\covcproj+1)}{\delta_{0}}\right) \right)\right) + \Log\!\left(\frac{6\covcproj^2}{\delta_0}\right)\right).
\end{equation*}
Finally, simplifying the expression, and noting that the constant $c$ only appears in a logarithmic term,  
one can verify that for a sufficiently large choice of numerical 
constant $c$ (e.g., any $c \geq 2^{17}$ would suffice), the right hand side is at most
\begin{equation*}
\frac{c \covcproj^2}{m}\!\left( \vc\Log(\covcproj) + \Log\!\left(\frac{1}{\delta_{0}}\right) \right),
\end{equation*}
which extends the inductive hypothesis to $m_0=m$.
The result now follows by the principle of induction.
\end{proof}

\begin{proof}[of Theorem \ref{thm:upper-bound-second}]
Fix any target concept $\target \in \H$ and any distribution $\Px$ on $\X$.
We will argue that, for any finite labeled data set $S$ with $\target$ labels,
the proper learner $\alg_{\text{ERM}}(S)$
outputs $\hat{h} \in \H$ correct on $S$, 
and in the case that $S$ is $m_{0}$ i.i.d.\ training examples 
(with $\Px$ marginal distribution and $\target$ labels), 
then for any $\delta_{0} \in (0,1)$, 
with probability at least $1-\delta_{0}$,
the classifier $\hat{h} = \alg(S)$ satisfies 
\begin{equation}
\label{eqn:the-claim-2}
\er(\hat{h}) \leq \frac{c \covcproj^3}{m_{0}}\!\left( \vc\Log(\covcproj) + \Log\!\left(\frac{1}{\delta_{0}}\right) \right),
\end{equation}
where $c \geq 1$ is an appropriate finite numerical constant.
Note that Theorem~\ref{thm:upper-bound-second} would immediately follow from this 
(since it holds for any $\Px$ and any $\target \in \H$), 
taking $\delta_{0} = \delta$, and noting that $m_{0}$ of size proportional to the 
stated bound on $\SC_{\alg_{\text{ERM}}}(\eps,\delta)$ makes the right hand side of \eqref{eqn:the-claim-2} 
less than $\eps$.

If $m_{0} < \covcproj+1$, the algorithm returns $\hat{h} = {\rm ERM}(S)$ in Step 1, which 
is an element of $\H$ that is correct on $S$ by definition; 
furthermore, the inequality trivially holds in this case, as the right hand side is greater than $1$.
These values of $m_{0}$ will serve as our base case in an inductive argument. 
Now, for induction, suppose 
$m \geq \covcproj+1$,
and that for any $\delta_{0} \in (0,1)$ and $m_{0} < m$, 
for any correctly labeled data set $S$ of size $m_{0}$, 
the classifier $\hat{h}$ returned by $\alg_{\text{ERM}}(S)$ 
is in $\H$ and is correct on $S$, and in the case that $S$ is i.i.d.\ (with 
$\Px$ marginal distribution and $\target$ labels),
then with probability at least $1-\delta_{0}$ 
\eqref{eqn:the-claim-2} holds.

Next we extend this claim to $m_{0} = m$.
Consider a run of $\alg_{\text{ERM}}(S)$ with 
a correctly labeled finite data set $S$ of size $m$.
Consider the sets $S_i$ and classifiers $h_i$ as 
defined in the algorithm, and with a slight abuse 
of notation we also use $S_i$ to denote the unlabeled 
portion of $S_i$ (i.e., the points $x$ such that $(x,\target(x)) \in S_i$).
Define 
$h_{{\rm maj}}(x) = {\rm Majority}(h_1(x),\ldots,h_{\covcproj+1}(x))$ 
(breaking ties to favor label $-1$, say), 
and as before using the notation \eqref{eq:majoritylabel}, 
\[
\X_{0} = \mathcal{X}_{\{h_1,\ldots,h_{\covcproj+1}\}, \covcproj}.
\]
Since each $h_i$ is in $\H$ (by the inductive hypothesis), 
the classifier $\hat{h} = {\rm Proj}_{\H}(h_1,\ldots,h_{\covcproj+1})$ 
in Step 4 is well defined, 
and by definition, $\hat{h} \in \H$ and has 
$\hat{h}(x) = h_{{\rm maj}}(x)$ 
on every $x \in \X_{0}$.
Note that since every $(x,y) \in S$ 
is included in just one set $S_i$, and hence is 
in every set $\bigcup_{j' \neq j} S_{j'}$ except
$j=i$, 
and by the inductive hypothesis every 
$j \neq i$ has $h_j$ correct on $\bigcup_{j' \neq j} S_{j'}$,
we see that every $(x,y) \in S$ has $x \in \X_{0}$,
and $h_{{\rm maj}}(x) = y$, 
so that this extends the 
claim that $\hat{h}$ is in $\H$ and is correct on $S$
for the inductive proof, and all that remains is to extend the bound on the error rate to hold for $m_0 = m$.

Toward this end, consider the case that 
$S$ is an i.i.d.\ data set of size $m$ (with 
marginal distribution $\Px$ and $\target$ labels).
Fix any $\delta_{0} \in (0,1)$.
By the inductive hypothesis 
each $h_i$ is correct on $\bigcup_{j \neq i} S_j$ 
and has $h_i \in \H$.
We will follow a similar ``conditioning'' argument
to that used in the proof of Theorem~\ref{thm:upper-bound}.
As in that proof, define  
$\ER(h) = \{ x : h(x) \neq \target(x) \}$
for any classifier $h$.
Since $\hat{h}$ agrees with $h_{{\rm maj}}$ 
on~$\X_{0}$, we have 
\begin{equation}
\label{eqn:hat-ER-2}
\ER(\hat{h}) \subseteq 
(\X \setminus \X_{0}) \cup ( \X_{0} \cap \ER(h_{{\rm maj}} )).
\end{equation}
Furthermore, for any $x \in \X \setminus \X_{0}$, 
at least two values of $i$ have $h_i(x)$
different from the majority of the values  
$h_1(x),\ldots,h_{\covcproj+1}(x)$,
which means there are at least two classifiers 
predicting each label in $\Y$, 
and hence there are at least two classifiers 
$h_i$ with 
$h_i(x) \neq \target(x)$.
Furthermore, any $x$ with 
$h_{{\rm maj}}(x) \neq \target(x)$ 
certainly also has at least two $h_i$ classifiers 
with $h_i(x) \neq \target(x)$.
Therefore, the set on the right hand side 
of \eqref{eqn:hat-ER-2} is contained within
$\bigcup_{i,j : i \neq j} \ER(h_i) \cap \ER(h_j)$.
In particular, this implies  
\begin{equation}
\label{eqn:hat-er-union-bound-2}
\er(\hat{h}) = \Px(\ER(\hat{h})) 
\leq \Px\!\left( \bigcup_{i,j : i < j} \ER(h_i) \cap \ER(h_j) \right) 
\leq \sum_{i,j : i < j} \Px\!\left( \ER(h_i) \cap \ER(h_j) \right).
\end{equation}
The remainder of the proof will establish that each term $\Px(\ER(h_i) \cap \ER(h_j))$ is small
with high probability.

For any $i$, we have 
\begin{align}
& \P\left( |\ER(h_i) \cap S_i| < (1/2) \er(h_i) |S_i| 
\text{ and } \er(h_i) \geq \frac{8}{|S_i|}\ln\!\left(\frac{3(\covcproj+1)}{\delta_0}\right) \right) \notag
\\ & = 
\E\!\left[ \P\Big( |\ER(h_i) \cap S_i| < (1/2) \er(h_i) |S_i| \Big| \er(h_i) \Big) \ind\!\left[ \er(h_i) \geq \frac{8}{|S_i|}\ln\!\left(\frac{3(\covcproj+1)}{\delta_0}\right) \right] \right]. \label{eqn:ER-conditional-exp-bound}
\end{align}
Since $S_i$ is excluded from the training set $\bigcup_{j \neq i} S_j$ producing $h_i$,
we have that $S_i$ and $h_i$ are independent random variables.
Therefore, a multiplicative Chernoff bound implies that (almost surely)
\begin{equation*}
\P\Big( |\ER(h_i) \cap S_i| < (1/2) \er(h_i) |S_i| \Big| \er(h_i) \Big) \leq \exp(- \er(h_i) |S_i| / 8),
\end{equation*}
so that 
\eqref{eqn:ER-conditional-exp-bound} is at most 
$\frac{\delta_0}{3(\covcproj+1)}$.
In other words, with probability at least 
$1 - \frac{\delta_0}{3(\covcproj+1)}$, either 
\begin{equation}
\label{eqn:eri-small-2}
\er(h_i) < \frac{8}{|S_i|}\log\!\left(\frac{3(\covcproj+1)}{\delta_0}\right)
\end{equation}
or else
\begin{equation}
\label{eqn:ER-cap-S-LB}
|\ER(h_i) \cap S_i| \geq (1/2) \er(h_i) |S_i|.
\end{equation}
By the union bound, with probability at least 
$1-\frac{\delta_0}{3}$, every $i$ satisfies 
at least one of \eqref{eqn:eri-small-2} or \eqref{eqn:ER-cap-S-LB}.

For distinct $j, i$, 
since $h_j$ is correct on $\bigcup_{i' \neq j} S_{i'} \supseteq S_i$ 
it is certainly correct on $\ER(h_i) \cap S_i$.
Also note that the samples in $\ER(h_i) \cap S_i$ are 
conditionally i.i.d.\ given $h_i$ and $|\ER(h_i) \cap S_i|$, 
with conditional distribution $\Px(\cdot | \ER(h_i))$.
Thus, by the classic PAC bound for ERM 
\citep*{vapnik:74,blumer:89} (applied under 
the conditional distribution given 
$h_i$ and $|\ER(h_i) \cap S_i|$) 
and the law of total probability, 
with probability at least $1-\frac{\delta_0}{3\covcproj^2}$,
\begin{equation*}
\Px( \ER(h_j) | \ER(h_i) )
\leq \frac{2/\ln(2)}{|\ER(h_i) \cap S_i|}\left(\vc \Log\!\left(\frac{2e |\ER(h_i) \cap S_i|}{\vc}\right) + \Log\!\left(\frac{6\covcproj^2}{\delta_0}\right)\right).
\end{equation*}
By the union bound (over $i,j$ pairs with $i < j$, 
and combining with the event above) and 
monotonicity of $x \mapsto (1/x)\Log(ax)$,
with probability at least $1-\frac{2}{3}\delta_{0}$, 
every pair $i,j$ with $i < j$ have either \eqref{eqn:eri-small-2} 
or 
\begin{equation}
\label{eqn:conditional-ERM-ER}
\Px( \ER(h_j) | \ER(h_i) )
\leq \frac{4/\ln(2)}{\er(h_i)|S_i|}\left(\vc \Log\!\left(\frac{e\; \er(h_i)|S_i|}{\vc}\right) + \Log\!\left(\frac{6\covcproj^2}{\delta_0}\right)\right).
\end{equation}

Since each $h_i$ is correct on $\bigcup_{j \neq i} S_j$, which is itself an 
i.i.d.\ data set of size strictly smaller than $m$ 
and no smaller than $\covcproj \lfloor m / (\covcproj+1) \rfloor \geq \frac{m}{3}$ 
(using the fact that $\covcproj \geq 2$, from the definition), 
the inductive hypothesis and the union bound imply that 
with probability at least 
$1 - \frac{\delta_{0}}{3}$, 
every $h_i$ has
\begin{align*}
\er(h_i) &
\leq \frac{3c \covcproj^3}{m}\!\left( \vc\Log(\covcproj) + \Log\!\left(\frac{3 (\covcproj+1)}{\delta_{0}}\right) \right).
\end{align*} 
Plugging this into the log in \eqref{eqn:conditional-ERM-ER} above, together with the fact that 
$|S_i| \geq \lfloor m / (\covcproj+1) \rfloor \geq (1/2) m / (\covcproj+1)$, 
we have by the union bound that  
with probability at least $1-\delta_{0}$, every pair $i,j$ with $i < j$ either 
have \eqref{eqn:eri-small-2} or have that $\Px(\ER(h_j) | \ER(h_i) )$ is upper bounded by 
\begin{equation*}
\frac{(8/\ln(2))(\covcproj+1)}{\er(h_i)m}\left(\vc \Log\!\left(\frac{3}{2} e c \covcproj^2\!\left( \Log(\covcproj) + \frac{1}{\vc}\Log\!\left( \frac{6 \covcproj^2}{\delta_{0}} \right)  \right) \right) + \Log\!\left(\frac{6\covcproj^2}{\delta_0}\right)\right).
\end{equation*}
In either case (i.e., whether \eqref{eqn:eri-small-2} holds or not), 
on this event we have
\begin{align*}
& \Px( \ER(h_i) \cap \ER(h_j) ) 
= \er(h_i) \Px(\ER(h_j) | \ER(h_i))
\\ & \leq \frac{16(\covcproj+1)}{m}\left(\vc \Log\!\left(\frac{3}{2} e c \covcproj^2\!\left( \Log(\covcproj) + \frac{1}{\vc}\Log\!\left( \frac{6 \covcproj^2}{\delta_{0}} \right)  \right) \right) + \Log\!\left(\frac{6\covcproj^2}{\delta_0}\right)\right).
\end{align*}
Combining this with \eqref{eqn:hat-er-union-bound-2},
we conclude that on the above event of probability at least $1-\delta_{0}$, 
\begin{equation*}
\er(\hat{h}) \leq \binom{\covcproj+1}{2}
\frac{16(\covcproj+1)}{m}\left(\vc \Log\!\left(\frac{3}{2} e c \covcproj^2\!\left( \Log(\covcproj) + \frac{1}{\vc}\Log\!\left( \frac{6 \covcproj^2}{\delta_{0}} \right)  \right) \right) + \Log\!\left(\frac{6\covcproj^2}{\delta_0}\right)\right).
\end{equation*}
By simplifying the expression on the right hand side 
and noting that the constant $c$ only appears in a logarithm, 
one can verify that for a sufficiently large choice of numerical constant $c$
(e.g., any $c \geq e^8$ would suffice), 
the right hand side is at most
\begin{equation*}
\frac{c \covcproj^3}{m}\!\left( \vc\Log(\covcproj) + \Log\!\left(\frac{1}{\delta_{0}}\right) \right),
\end{equation*}
which extends the inductive hypothesis to $m_0=m$.
The result now follows by the principle of induction.
\end{proof}

\section{Proofs of lower bounds}
\label{sec:lowerboundproofs}

We start with the analysis of the 
\emph{coupon collector's problem} 
and corresponding lower bounds (see e.g., \citep*{motwani2010randomized}). 
Since we need a slightly more general result 
(not appearing in the standard textbooks to the best of our knowledge), 
we present a short proof for the sake of completeness. We remark that a similar argument was used in \citep*{simon:15}.

\begin{lemma}[Generalized coupon collector's problem]
\label{lem:couponcollect}
Let $m\leq k\in\mathbb{N}$. 
Consider a sequence $x_1,x_2,\ldots$ of independent uniform draws from a set of size $k$.
Assume $z \in\mathbb{N}$ satisfies that with probability at least $1/2$, 
the number of distinct elements among $x_1,\ldots, x_z$ is at least $k-m$.
Then,
\[
z\geq k\left(\ln\frac{k}{m} - 1 - \sqrt{\frac{2}{m}}\right).
\]
\end{lemma}
\begin{proof}
Let $Z$ denote the random variable that counts the number of independent draws until at least $k - m$ distinct elements are present among $x_1,x_2,\ldots$. We may write $Z = \sum\limits_{i = 1}^{k - m}Z_i$, 
where $Z_i$ represents the (random) number of draws 
after $i-1$ distinct elements were observed 
and up to and including the first draw 
when $i$ distinct elements have been observed.
Observe that $Z_i$-s are independent, each having the geometric distribution with parameter $p_i = \frac{k - i + 1}{k}$. Thus, $\E Z_i = \frac{1}{p_i}$ and $\Var(Z_i) = \frac{1 - p_i}{p_i^2}$, which implies
\[
\E Z = \sum\limits_{i = 1}^{k - m} \E Z_i = \sum\limits_{i = 1}^{k - m}\frac{k}{k - i + 1} = k(H_k - H_m),
\]
where $H_p = \sum\limits_{i = 1}^p\frac{1}{i}$ stands for the $p$-th Harmonic number. Further, we have
\[
\Var(Z) = \sum\limits_{i = 1}^{k - m}\frac{k(i - 1)}{(k - i + 1)^2} 
= \sum\limits_{j = m + 1}^{k}\frac{k(k - j)}{j^2} \le k^2\sum\limits_{j = m + 1}^{k}\frac{1}{j^2} \le \frac{k^2}{m}.
\]
Finally, by Chebyshev's inequality and the relation $\ln p \le H_p \le \ln p + 1$ we have, with probability at least $\frac{1}{2}$,
\[
Z \ge \E Z - k\sqrt{\frac{2}{m}} \ge k\left(\ln\frac{k}{m} - 1 - \sqrt{\frac{2}{m}}\right).
\]
The claim follows.
\end{proof}
\begin{remark}
\label{rem:handycor}
We will often use the following handy corollary of Lemma \ref{lem:couponcollect}: 
under the conditions of this result,  
$z \geq \frac{k}{2}\ln \frac{k}{m}$,
provided that $1 + \sqrt{\frac{2}{m}} \le \frac{1}{2}\ln \frac{k}{m}$.
\end{remark}

\begin{proof}[of Theorem \ref{thm:lower-bound}]
A lower bound $\Omega\!\left(\frac{\vc}{\epsilon} + \frac{1}{\epsilon}\log\!\left(\frac{1}{\delta}\right)\right)$ 
holds for all learning algorithms \citep*{vapnik:74,ehrenfeucht:89}, 
so we focus only on establishing a lower bound $\Omega\!\left( \frac{1}{\epsilon} \log(\covcmod) \ind[ \epsilon \leq 1/\covcmod ] \right)$ for proper learners. Without loss of generality we may assume that $\covcmod \ge 128$ since for smaller values of $\covcmod$ the lower bound is automatically established 
by choosing a small enough numerical constant factor.
We in fact establish a stronger result, which also implies the second claim: 
namely, for any $k \geq 128$ such that 
there exists a hollow star of size $k$, 
any $\epsilon \leq 1/k$ has $\SC_{{\rm prop}}(\epsilon,\delta) \geq \frac{c}{\epsilon} \log( k )$
(for a numerical constant $c > 0$).
Note that both of the claimed lower bounds will follow from this, since when $\covcmod < \infty$ 
there exists a hollow star of size $\covcmod$,
and when $\covcmod = \infty$ there exists a sequence $k_i \to \infty$ for which 
there exist hollow stars of each size $k_i$, so that choosing $\epsilon_i = 1/k_i$ 
the lower bound $\frac{c}{\epsilon_i} \log\!\left( \frac{1}{\epsilon_i} \right)$ holds for 
each $\epsilon_i$.

Fix any $k \geq 128$ such that 
there exists a hollow star set  $S = \{(x_1,y_1),\ldots,(x_{k},y_{k})\}$, 
and for each $i \in \{1,\ldots,k\}$ let 
$h_i \in \H$ be such that 
$\{ j : h_i(x_j) \neq y_j \} = \{ i \}$.
Fix any proper learning algorithm $\alg'$. 
We construct a target function $\target$ and distribution $\Px$ 
to witness the lower bound via the probabilistic method.
Let $\epsilon \leq 1/k$ 
and choose $i^* \sim {\rm Uniform}(\{2,\ldots,k\})$,  
and set 
$\target = h_{i^*}$ and 
$\Px(\{x_i\}) = \epsilon/(1-\epsilon)$ for 
$i \in \{2,\ldots,k\}\setminus\{i^*\}$, 
$\Px(\{x_1\}) = 1-(k-2)\epsilon/(1-\epsilon)$ (which is greater than $\epsilon$), and $\Px(\{x_{i^*}\})=0$.
Consider running $\alg'$ with a 
data set $D_n$ 
(conditionally i.i.d.\ given $i^*$, 
with each point $(X,Y)$ having 
$X \sim \Px$ and $Y=\target(X)$), 
of some size  
$n < \frac{1}{8} \frac{1-\eps}{\eps} \ln(k-2)$,
and let $\hat{h}$ be the classifier 
it outputs.  Since $\alg'$ is proper 
($\hat{h} \in \H$) and $\H[S]=\emptyset$ 
($S$ being a hollow star), we know that 
$\hat{h}$ cannot realize 
the $y_i$ classification of every $x_i$, 
so there must be a non-empty set 
$\hat{I} = \{ i : \hat{h}(x_{i}) \neq y_{i} \} \neq \emptyset$.
If any of these $\hat{i} \in \hat{I}$ are not equal $i^*$, then $\er(\hat{h}) > \epsilon$.
Denote by $\hat{n}$ the number of the 
$n$ data points in $D_n$ falling in $\{(x_2,y_2),\ldots,(x_k,y_k)\} \setminus \{(x_{i^*},y_{i^*})\}$,
and denote by $\hat{n}_1$ the number 
of \emph{distinct} elements of 
$\{(x_2,y_2),\ldots,(x_k,y_k)\} \setminus \{(x_{i^*},y_{i^*})\}$ 
observed in the data set $D_n$.
By Markov's inequality, 
with probability at least 
 $\frac{3}{4}$, 
 we have 
$\hat{n} \leq 4\frac{n(k - 2)\eps}{1 - \eps} 
< \frac{1}{2} (k-2) \ln (k-2)$.
Furthermore, note that conditioned 
on $\hat{n}$ and $i^*$, the $\hat{n}$ 
samples in $D_n$ falling in 
$\{ (x_2,y_2),\ldots,(x_k,y_k) \} \setminus \{(x_{i^*},y_{i^*})\}$ 
are conditionally independent, 
with conditional distribution uniform 
on this set.
Therefore, on the event that 
$\hat{n} < \frac{1}{2} (k-2) \ln (k-2)$, 
Lemma~\ref{lem:couponcollect} implies that 
$\P( \hat{n}_1 < k-3 | \hat{n}, i^* ) > \frac{1}{2}$ 
(noting that 
$k \geq 128$ implies $1+\sqrt{2} \leq \frac{1}{2}\ln(k-2)$).
Finally, note that conditioned on 
$D_n$, the variable $i^*$ has 
conditional distribution uniform 
on the $k-1 - \hat{n}_1$ values 
$i \in \{2,\ldots,k\}$ with  
$(x_i,y_i) \notin D_n$.
Thus, on the event that $\hat{n}_1 < k-3$, 
we have that 
$\P( \hat{I} \neq \{i^*\} | \hat{I}, D_n ) 
\geq \frac{(k-1 - \hat{n}_1)-1}{k-1-\hat{n}_1} \geq \frac{1}{2}$.
Altogether, we have 
\begin{align*}
& \P\!\left( \hat{I} \neq \{i^*\} \right) 
\geq \E\!\left[ \P\!\left( \hat{I} \neq \{i^*\} \middle| \hat{I}, D_n \right) \ind\!\left[ \hat{n}_1 < k-3 \right] \right] 
\geq \tfrac{1}{2} \P\!\left( \hat{n}_1 < k-3 \right) 
\\ & \geq \tfrac{1}{2} \E\!\left[ \P\!\left( \hat{n}_1 < k-3 \middle| \hat{n}, i^* \right) \ind\!\left[ \hat{n} < \tfrac{1}{2} (k-2) \ln (k-2) \right] \right] 
\geq \tfrac{1}{4} \P\!\left( \hat{n} < \tfrac{1}{2} (k-2) \ln (k-2) \right)
\geq \tfrac{3}{16}.
\end{align*}

In particular, this implies that for any proper learning algorithm $\alg'$, 
if $n < \frac{1}{8}\frac{1-\epsilon}{\epsilon} \ln(k-2)$, 
there exist fixed choices of $\target \in \H$ and $\Px$ such that, 
with probability at least $3/16$, 
the classifier $\hat{h}$ returned by $\alg'$  has $\er(\hat{h}) > \epsilon$. 
The claim follows.
\end{proof}
\begin{proof}[of Theorem \ref{thm:sometimes-tight}]
Fix any $\vc$, $\covc$.
Since there is already a known 
lower bound 
$\frac{c'}{\epsilon}\!\left( \vc + \Log\!\left(\frac{1}{\delta}\right)\right)$ 
from \citep*{ehrenfeucht:89,blumer:89,vapnik:74},
for PAC learning in general 
(for some numerical constant $c'>0$), 
we focus on showing a lower bound 
$\frac{c\vc}{\epsilon}\Log\!\left(\frac{\covc}{\vc} \land \frac{1}{\epsilon}\right)$
for some numerical constant $c>0$.
Furthermore, since $\Log(x) \geq 1$ 
(from its definition above), 
this lower bound again already follows from 
the lower bound of \citet*{ehrenfeucht:89}  
in the case that $\covc < 126 \vc$ 
(for instance, taking $\X = \{1,\ldots,\vc-1+\covc\}$ 
and $\H = \{ x \mapsto 2 \ind[ x \in I ] - 1 : I \subseteq \X, |I \cap \{1,\ldots,\covc\}|=1 \}$, 
which one can easily verify has VC dimension $\vc$ and dual Helly number $\covc$).
To address the remaining case,  
suppose $\covc \geq 126 \vc$.

In the special case $\vc = 1$, 
simply take $\X = \{1,\ldots,\covc\}$ 
and $\H = \{ x \mapsto 2 \ind[ x = t ] - 1 : t \in \X \}$ 
the singleton classifiers.  
It is an easy exercise to verify that the 
VC dimension of $\H$ is indeed $1$, and 
that $\{(x,-1) : x \in \X \}$ is a hollow star 
set of size $\covc$, so that Lemma~\ref{lem:all-the-hellys} (together with the fact that this is clearly the largest possible hollow star set, and that $\H$ is finite and 
therefore closed) implies the dual Helly number is 
indeed $\covc$.  The claimed lower bound 
for this case then follows from 
Theorem~\ref{thm:lower-bound}.
To address the remaining case, for the rest of the 
proof suppose $\vc \geq 2$.

Let 
$\X = \{ (i,j) : i \!\in\! \{\vc-1,\ldots,\covc+\vc-2\}, j \!\in\! \{ 1,\ldots,i \} \}$.
For each $i \in \{\vc-1,\ldots,\covc+\vc-2\}$ 
and each $J \subseteq \{ 1, \ldots, i \}$ with $|J|=\vc-1$, 
define a classifier $h_{i,J}(i',j) = 1 - 2\ind[ i' = i ] \ind[ j \notin J ]$: 
that is, $h_{i,J}$ classifies as $1$ everything that does not have first coordinate equal $i$, 
and exactly $\vc-1$ of the points that do have first coordinate equal $i$.
Set 
\[
\H = \{ h_{i,J} : i \in \{\vc-1,\ldots,\covc+\vc-2\}, J \subseteq \{1,\ldots,i\}, |J|=\vc-1 \}.
\]

We first show that the VC dimension of $\H$ is 
indeed $\vc$.  To see that $\vc$ points can be shattered, 
simply take $i=2\vc-1$ (which has $i \leq \covc+\vc-2$ since $\covc \geq \vc+1$) 
and we claim that the points $(i,1),\ldots,(i,\vc)$ 
are shattered: for any strict subset 
$J \subset \{1,\ldots,\vc\}$, we can realize 
a labeling with 
$\{(i,j) : j \in J\}$ positive and the other $\vc-|J|$ negative with $h_{i,J \cup J'}$ where $J'$ is 
any subset of $\{\vc+1,\ldots,2\vc-1\}$ with $|J'|=\vc-1-|J|$; 
also, we can realize the all-positive labeling of these 
$\vc$ points with $h_{\vc-1,1:(\vc-1)}$.
To show no set of $\vc+1$ points can be shattered, 
note that if $\{((i_1,j_1),-1),\ldots,((i_{\vc+1},j_{\vc+1}),-1)\}$ 
is realizable, then $i_{1} = \cdots = i_{\vc+1}$; 
but in this case, $\{((i_{1},j_{1}),1),\ldots,((i_{\vc},j_{\vc}),1),((i_{\vc+1},j_{\vc+1}),-1)\}$ is not realizable, 
and hence no set of size $\vc+1$ is shattered.

Next we argue that the dual Helly numer of $\H$ is 
indeed $\covc$.  To see that it is at least $\covc$, 
note that $\{((\covc+\vc-2,1),-1),\ldots,((\covc+\vc-2,\covc),-1)\}$ is a hollow star set of size $\covc$, 
so that Lemma~\ref{lem:all-the-hellys} implies 
the dual Helly number is at least $\covc$.
To see that it is also at most $\covc$, 
consider any unrealizable set $S$.  
If some $x$ has $\{(x,-1),(x,1)\} \subseteq S$, 
this is clearly an unrealizable subset of size 
$2 \leq \covc$.  Otherwise if no such $x$ exists, 
then note that since $h_{\vc-1,1:(\vc-1)}$ is positive 
on all of $\X$, there must be some $(i,j_{-})$ with 
$((i,j_{-}),-1) \in S$.  If there are in fact 
\emph{two} points $(i,j)$, $(i',j')$ with $i \neq i'$ 
and $\{((i,j),-1),((i',j'),-1)\} \subseteq S$, then
again this is an unrealizable subset of size $2 \leq \covc$.  Otherwise, if every $(i,j)$ with $((i,j),-1) \in S$ 
has the \emph{same} $i$, then it must be that 
either there exist $j_1,\ldots,j_{\vc}$ 
with $\{ ((i,j_1),1),\ldots,((i,j_{\vc}),1) \} \subset S$, 
in which case $\{ ((i,j_1),1),\ldots,((i,j_{\vc}),1),((i,j_{-}),-1)\}$ 
is an unrealizable subset of size $\vc+1 \leq \covc$, 
or else there exist $j_1,\ldots,j_{i-(\vc-2)}$ 
with $\{ ((i,j_1),-1),\ldots,((i,j_{i-(\vc-2)}),-1) \} \subseteq S$, 
in which case this is an unrealizable subset of size 
$i-(\vc-2) \leq \covc$.
Since this covers all possible cases for the set $S$, 
we conclude that the dual Helly number is equal  
$\covc$.

Fix any $\delta \in (0,1/100)$.
For any $\epsilon \in (1/504,1/8)$, 
a lower bound $\frac{c \vc}{\epsilon} \Log\!\left(\frac{\covc}{\vc} \land \frac{1}{\epsilon} \right)$ follows from the lower bound 
$\frac{c' \vc}{\epsilon}$ of \citet*{ehrenfeucht:89} 
(for $c$ a sufficiently small numerical constant).  
To address the remaining case, fix any $\epsilon \in (0,1/504]$.
If $\epsilon \geq \frac{\vc-1}{4(\covc-1)}$,
let $i_{\epsilon} = \lfloor (\vc-1)/(4\epsilon) \rfloor + \vc-1$, 
and otherwise let $i_{\epsilon} = \covc+\vc-2$.
We prove the lower bound via the probabilistic method.
Let $J^*$ be a subset of $\{1,\ldots,i_{\epsilon}\}$ 
with $|J^*| = \vc-1$ 
chosen uniformly at random (without replacement).
Let $\Px(\{(i_{\epsilon},j)\}) = \frac{4\epsilon}{\vc-1}$
for every $j \in \{1,\ldots,i_{\epsilon}\} \setminus J^*$, and let $\Px(\{(\vc-1,1)\})=1-(i_{\epsilon}-(\vc-1)) \frac{4\epsilon}{\vc-1}$,  
and define the target concept $\target = h_{i_{\epsilon},J^*}$: 
in particular, 
$\target((i_{\epsilon},j)) = 2 \ind[ j \in J^* ] - 1$,
and hence $\Px$ has zero mass on the set 
of all points $(i_{\epsilon},j)$ 
where $\target((i_{\epsilon},j)) = 1$ 
and has mass $\frac{4\epsilon}{\vc-1}$ 
on every point $(i_{\epsilon},j)$ where 
$\target((i_{\epsilon},j))=-1$; 
any remaining probability mass is placed 
on $(\vc-1,1)$, which is an uninformative 
point (since every $h_{i,J}$ classifies it $1$).

Fix any sample size 
$n \in \nats$ 
with 
$n < \frac{\vc-1}{32 e \epsilon} \ln\!\left( \frac{1}{\vc-1}\min\!\left\{ \left\lfloor \frac{\vc-1}{4\epsilon} \right\rfloor, \covc - 1 \right\} \right)$,
fix any proper learning algorithm $\alg'$, 
and let $\hat{h}$ be the classifier returned by running $\alg'$ 
on a conditionally 
i.i.d.\ (given $J^*$) training set $D_n$ of size $n$ 
(with each $(X,Y) \in D_n$ having $X \sim \Px$ and 
$Y=\target(X)$ given $J^*$).  
Let $\Z_{i_{\epsilon}} = \{ ((i_{\epsilon},j),-1) : j \in \{1,\ldots,i_{\epsilon}\} \setminus J^* \}$
and $\hat{n} = | D_n \cap \Z_{i_{\epsilon}} |$,
and note that we have $\E[ \hat{n} | J^* ] 
= \frac{4\epsilon}{\vc-1} (i_{\epsilon}-(\vc-1)) n$.
Thus, by a Chernoff bound and the law of total probability, 
with probability at least $1/2$, 
it holds that $\hat{n} \leq 1 + 2 e \E[ \hat{n} | J^* ] 
= 1 + \frac{8 e \epsilon}{\vc-1} (i_{\epsilon}-(\vc-1)) n$
\citep*[see][]{motwani2010randomized}.
Combining this with the constraint on $n$, 
on this event we have 
$\hat{n}
< 1 + \frac{i_{\epsilon}-(\vc-1)}{4}\ln\!\left(  \frac{1}{\vc-1} \min\!\left\{ \left\lfloor \frac{\vc-1}{4\epsilon} \right\rfloor, \covc - 1 \right\} \right) 
\leq \frac{i_{\epsilon}-(\vc-1)}{2} \ln\!\left( \frac{i_{\epsilon}-(\vc-1)}{\vc-1}\right)$ 
(using the fact that $\epsilon \leq 1/504$).
Furthermore, note that the samples in $D_n \cap \Z_{i_{\epsilon}}$ 
are conditionally i.i.d.\ ${\rm Uniform}(\Z_{i_{\epsilon}})$ 
given $\hat{n}$.
Also note that the assumptions that 
$\covc \geq 126 \vc$ 
and $\epsilon \leq 1/504$ imply 
$i_{\epsilon}-(\vc-1) \geq 126 (\vc-1)$, 
so that 
$1 + \sqrt{\frac{2}{\vc-1}} \leq \frac{1}{2}\ln \frac{i_{\epsilon}-(\vc-1)}{\vc-1}$.
Therefore, denoting by $\hat{n}_1$ the number of \emph{distinct} 
elements of $\Z_{i_{\epsilon}}$ present in $D_n$,
we have that, on the event that 
$\hat{n} < \frac{i_{\epsilon}-(\vc-1)}{2}\ln\!\left(\frac{i_{\epsilon}-(\vc-1)}{\vc-1} \right)$,
Lemma~\ref{lem:couponcollect} implies 
$\P( \hat{n}_1 < i_{\epsilon}-2(\vc-1) | \hat{n},J^* ) > \frac{1}{2}$.

Since $\alg'$ is a proper learning algorithm, 
it must be that $\hat{h} = h_{\hat{i},\hat{J}}$ for some 
$\hat{i} \in \{\vc-1,\ldots,\covc+\vc-2\}$ and 
$\hat{J} \subseteq \{1,\ldots,\hat{i}\}$ with $|\hat{J}|=\vc-1$.
If $\hat{i} \neq i_{\epsilon}$, then 
$\er(\hat{h}) = (i_{\epsilon} - (\vc-1)) \frac{4\epsilon}{\vc-1} 
> \epsilon$.
Otherwise, suppose $\hat{i} = i_{\epsilon}$.
Then 
$\er(\hat{h}) = | \hat{J} \setminus J^* | \frac{4\epsilon}{\vc-1}$.
Note that, conditioned on $D_n$, the variable $J^*$ has 
conditional distribution uniform on the subsets 
of the (size $i_{\epsilon}-\hat{n}_1$) set 
$\{ j \in \{1, \ldots, i_{\epsilon}\} : ((i_{\epsilon},j),-1) \notin D_n \}$ of size $\vc-1$.  In particular, on the event $\hat{i} = i_{\epsilon}$, 
we have $\E\Big[ |\hat{J} \setminus J^*| \Big| D_n, \hat{J}, \hat{i} \Big] 
\geq (\vc-1)\frac{i_{\epsilon} - \hat{n}_1 - (\vc-1)}{i_{\epsilon}-\hat{n}_1}$.
On the event that 
$\hat{n}_1 < i_{\epsilon}-2(\vc-1)$, 
this implies 
$\E\Big[ |\hat{J} \setminus J^*| \Big| D_n, \hat{J}, \hat{i} \Big] 
> \frac{\vc-1}{2}$.
Therefore, a Chernoff bound (for sampling without replacement; 
see \citealp*{hoeffding:63}) implies that, 
on the events that $\hat{n}_1 < i_{\epsilon}-2(\vc-1)$ 
and $\hat{i} = i_{\epsilon}$, 
we have 
$\P\!\left( |\hat{J} \setminus J^*| 
\leq \frac{\vc-1}{4} \middle| D_n, \hat{J}, \hat{i} \right) \leq \exp\!\left\{ - \frac{\vc-1}{16} \right\} 
\leq e^{-1/16} < 1-\frac{1}{17}$.
In particular, note that 
if $\hat{i} = i_{\epsilon}$ and 
$|\hat{J} \setminus J^*| > \frac{\vc-1}{4}$, 
then 
$\er(\hat{h}) > \epsilon$.
Altogether, we have that 
\begin{align*}
\P( \er(\hat{h}) > \epsilon ) 
& \geq \P( \hat{i} \neq i_{\epsilon} ) + 
\E\!\left[ \P\Big( |\hat{J} \setminus J^*| > \frac{\vc-1}{4} \Big| D_n, \hat{J}, \hat{i} \Big) \ind[ \hat{n}_1 < i_{\epsilon} - 2(\vc\!-\!1) ] \ind[ \hat{i} = i_{\epsilon} ] \right] 
\\ & \geq 
\frac{1}{17} \E\!\left[ \P\Big( \hat{n}_1 < i_{\epsilon} - 2(\vc-1) \Big| \hat{n}, J^* \Big) \ind\!\left[ \hat{n} < \frac{i_{\epsilon}-(\vc-1)}{2}\ln\!\left(\frac{i_{\epsilon}-(\vc-1)}{\vc-1} \right) \right] \right]
\\ & \geq \frac{1}{34} \P\!\left( \hat{n} < \frac{i_{\epsilon}-(\vc-1)}{2}\ln\!\left(\frac{i_{\epsilon}-(\vc-1)}{\vc-1} \right) \right) 
\geq \frac{1}{68} > \delta.
\end{align*}
In particular, this implies that there 
exists a non-random choice of $\target \in \H$
and $\Px$ such that, with probability 
strictly greater than $\delta$, 
it holds that $\er(\hat{h}) > \epsilon$.  The claimed 
lower bound on $\SC_{{\rm prop}}(\epsilon,\delta)$ 
follows by simplifying the expression of the 
constraint on $n$ above 
(which is lower-bounded by the 
expression in the theorem, for a sufficiently 
small choice of the numerical constant $c$).

For the final claim in the theorem, it is clear that we can extend the above construction 
to an infinite space by allowing all 
$i \in \nats$ with $i \geq \vc-1$, 
in which case $\covcmod = \covc = \infty$ 
(since there exist hollow star sets of 
unbounded sizes, following 
the same argument given above), 
and the $\vc \Log\!\left( \frac{\covc}{\vc} \land \frac{1}{\epsilon} \right)$ term simplifies to 
$\vc \Log\!\left( \frac{1}{\epsilon} \right)$.
\end{proof}

\section{Proof of Theorem \ref{thm:samplecomp}}
\label{app:samplecomp}

The essence of the proof of this result is in fact very simple,
relying only on one technical construction: a set system on the data indices.
Specifically, for any $m \in \nats$, 
consider a family $\mathcal{I}_{m}$ of subsets of $\{1,\ldots,m\}$
satisfying the following two properties, for some $T_m \in \{1,\ldots,m\}$:
\begin{itemize}
    \item[(i)] each $I\in\mathcal{I}_m$ has size $|I|\leq m - T_m$,
    \item[(ii)] for every $i_1,i_2,\ldots i_\ell \in \{1,\ldots,m\}$ there exists $I\in\mathcal{I}_m$ 
    such that $\{i_1,i_2,\ldots i_\ell\} \subseteq I$.
\end{itemize}

Let $(\kappa,\rho)$ be a stable compression 
scheme of size $\ell$.
Fix any distribution $\Px$, any $\target \in \H$, and any 
$\delta \in (0,1)$, and let  
$S = (X_{1:m},\target(X_{1:m}))$ be such 
that $X_{1:m} \sim \Px^m$.
Given any family $\mathcal{I}_m$ satisfying (i) and (ii), we will establish that with 
probability at least $1-\delta$, 
\begin{equation}
\label{eqn:abstract-stable-compression-bound}
\er(\rho(\kappa(S))) \leq \frac{1}{T_m} \left( \ln(|\mathcal{I}_m|) + \ln\!\left(\frac{1}{\delta}\right) \right).
\end{equation}

As a simple example of such a family $\mathcal{I}_m$ that yields Theorem~\ref{thm:samplecomp},
consider any partition of $\{1,\ldots,m\}$ into 
disjoint blocks $I_1,\ldots,I_{2\ell}$,
each of size either $\lceil m/(2\ell) \rceil$ or $\lfloor m/(2\ell) \rfloor$.
Then we can define  
\begin{equation}
\label{eqn:naive-cover-system}
\mathcal{I}_m = \left\{ \bigcup \{ I_j : j \in \mathcal{J} \} : \mathcal{J} \subseteq \{1,\ldots,2\ell\}, |\mathcal{J}|=\ell \right\}.
\end{equation}
This clearly satisfies the above properties, 
with $T_m = \ell \lfloor m/(2\ell) \rfloor$, 
and has size $|\mathcal{I}_m| = \binom{2\ell}{\ell} < 4^{\ell}$, 
and hence plugging this into \eqref{eqn:abstract-stable-compression-bound} 
yields the bound stated in Theorem~\ref{thm:samplecomp}.
We now finish the proof of Theorem~\ref{thm:samplecomp} 
by establishing the bound \eqref{eqn:abstract-stable-compression-bound}.

Fix any family $\mathcal{I}_m$ satisfying (i) and (ii) above.
For the set $S$ as introduced above, 
for any $I \subseteq \{1,\ldots,m\}$ define $S_I = \{(X_i,\target(X_i)) : i \in I\}$.
For any $I \in \mathcal{I}_m$, 
since $S_{(1:m) \setminus I}$ is independent of $S_{I}$, 
and property (i) implies $|S_{(1:m) \setminus I}| \geq T_m$, we have 
\begin{equation*}
\P\Big( \rho(\kappa(S_I)) \text{ is correct on } S_{(1:m) \setminus I} \text{ and } \er(\rho(\kappa(S_I))) > \epsilon \Big) 
\leq (1-\epsilon)^{T_m}.
\end{equation*}
However, by property (ii) 
there must exist at least one $I^* \in \mathcal{I}$ with 
$\kappa(S) \subseteq S_{I^*}$, 
which means $\rho(\kappa(S_{I^*})) = \rho(\kappa(S))$ (by the stability property).
Thus, since $\rho(\kappa(S))$ is correct on 
all of $S$ (because $(\kappa,\rho)$ is a valid 
compression scheme), including 
$S_{(1:m) \setminus I^*}$, 
we have for this (data-dependent) 
choice of $I^*$ that  
$\rho(\kappa(S_{I^*}))$ 
is correct on $S_{(1:m) \setminus I^*}$.
Therefore, by basic inequalities and a union bound, 
\begin{align}
\P\!\left( \er(\rho(\kappa(S))) > \epsilon \right) 
& = \P\!\left( \er(\rho(\kappa(S_{I^*}))) > \epsilon \right) \notag 
\\ & \leq \P\!\left( \exists I \in \mathcal{I}_{m} : 
\rho(\kappa(S_I)) \text{ is correct on } S_{(1:m) \setminus I} \text{ and } \er(\rho(\kappa(S_I))) > \epsilon \right) \notag 
\\ & \leq |\mathcal{I}_{m}| (1-\epsilon)^{T_m}
\leq |\mathcal{I}_{m}| e^{-\epsilon T_m }. \label{eqn:comp-bound}
\end{align}
In particular, for any $\delta \in (0,1)$, choosing $\epsilon$ equal the expression 
on the right hand side of \eqref{eqn:abstract-stable-compression-bound}
makes the rightmost expression in \eqref{eqn:comp-bound} equal $\delta$, 
which therefore completes the proof of the abstract bound \eqref{eqn:abstract-stable-compression-bound}.
The bound in Theorem~\ref{thm:samplecomp} follows by plugging in the 
family $\mathcal{I}_m$ from \eqref{eqn:naive-cover-system}, 
which has $|\mathcal{I}_m| < 4^\ell$ and $T_m = \ell \lfloor m/(2\ell) \rfloor > (m-2\ell)/2$.
\hfill $\BlackBox$

\end{document}